\documentclass{article}
\usepackage[a4paper,margin=3cm]{geometry}

\usepackage{cmap}             
\usepackage[T1]{fontenc}      
\usepackage[utf8]{inputenc}
\usepackage[ngerman,english,shorthands=none]{babel}

\usepackage{amsmath,amssymb}
\usepackage[table]{xcolor}
\definecolor{themecolor}{rgb}{0,0.57,0.76}       
\definecolor{subthemecolor}{rgb}{0.64,0,0.48}    
\definecolor{gray}{rgb}{0.6,0.6,0.6}

\usepackage[most]{tcolorbox}
\tcbuselibrary{breakable, skins}
\newtcolorbox[auto counter, number within=section]{definitionbox}[2][]{%
    colback=themecolor!10!white, 
    colframe=themecolor, 
    coltitle=white, 
    fonttitle=\bfseries, 
    title=#2,
    #1 
}

\tcbuselibrary{listings}
\tcbuselibrary{xparse}

\usepackage{cite}  

\usepackage{url}            
\usepackage{booktabs}       
\usepackage{dcolumn}
\usepackage{makecell}
\usepackage{multirow}
\usepackage{amsfonts}       
\usepackage{nicefrac}       
\usepackage{microtype}      
\usepackage{graphicx}
\usepackage{doi}
\usepackage[labelfont=bf]{subcaption}
\usepackage{algorithm}
\usepackage{algpseudocode}
\algrenewcommand{\algorithmiccomment}[1]{\hfill\textcolor{gray}{$\triangleright$ #1}}

\usepackage{placeins}        
\usepackage{fontawesome5}    

\usepackage[
	colorinlistoftodos,		
	\languagename,			
	textsize=footnotesize,	
]{todonotes}

\DeclareCaptionFont{themesf}{\color{themecolor}\sffamily}
\DeclareCaptionFont{themebf}{\color{themecolor}\bfseries}
\usepackage{standalone}
\usepackage{tikz,pgfplots,pgfplotstable}
\pgfplotsset{compat=newest}
\usetikzlibrary{positioning, fit, backgrounds, shapes, calc}

\usepackage{amsthm}
\usepackage{aliascnt}
\usepackage{hyperref}
\hypersetup{%
   colorlinks = true,
   linkcolor = subthemecolor,
   citecolor = subthemecolor,
   urlcolor = subthemecolor}

\newtheorem{definition}{Definition}[section]

\newaliascnt{lemma}{definition}
\newtheorem{lemma}[lemma]{Lemma}
\aliascntresetthe{lemma}

\newaliascnt{theorem}{definition}
\newtheorem{theorem}[theorem]{Theorem}
\aliascntresetthe{theorem}

\newaliascnt{corollary}{definition}
\newtheorem{corollary}[corollary]{Corollary}
\aliascntresetthe{corollary}

\addto\extrasenglish{%

}
   
\newcommand{\solver}[1]{\texttt{#1}}
\newcommand{\Fld}{\texttt{fld}}
\newcommand{\Move}{\texttt{move}}
\newcommand{\TchFld}{\texttt{touch}}
\newcommand{\Tmin}{\texttt{T}_\text{min}}
\newcommand{\cupdot}{\mathbin{\mathaccent\cdot\cup}}

\title{Solving Flood-It Puzzles: A Comparison of Tailored A* Search with Declarative Solver Approaches}

\author{Philipp von Manteuffel \\[0.2em]
        \small \faEnvelope\ \texttt{philipp.von.manteuffel@komail.eu} \\[0.8em]
        Benjamin Blankertz \\[0.2em]
        \small \faEnvelope\ \texttt{benjamin.blankertz@tu-berlin.de} \\[0.8em]
        \small Chair for Neurotechnology \\
        \small Technische Universität Berlin
        }

\hypersetup{
pdftitle={Flood-It},
pdfauthor={Philipp von Manteuffel, Benjamin Blankertz},
pdfkeywords={Flood-It, Puzzles, A* Search, Heuristics, SAT, Integer Programming, Constraint Programming, Minimization Problem},
}

\begin{document}
\maketitle

\begin{abstract}
We present an optimized approach for finding minimal move sequences to solve Flood-It puzzles using A* search and compare it with approaches using SAT, Integer Programming, and Constraint Programming solvers. The performance of A* search is considerably improved by using an appropriate heuristic and sophisticated equivalences of game states to limit the search space. Each aspect is systematically evaluated, providing insights that generalize to other problems. By comparing the best variant of each approach, the A* search is the fastest in finding a minimal solution in about 75\% of the Flood-It instances tested. This result is in line with our expectation since A* specifically looks for a minimum solution. In contrast, the SAT-based approach is designed to solve the decision problem of whether a solution with a given length exists. An iterative search is required to determine a minimum solution length using a SAT-based approach. However, with increasing complexity, quantified by the minimum solution length, the SAT-based approach outperforms A*. We also compared the different approaches to the decision problem, where A* again performed best for the largest number of puzzles. Here, again, SAT supersedes A* for higher complexities, albeit with a lower crossing point.
\end{abstract}

\maketitle

\newpage

\section{Introduction}
\label{sec:intro}
Flood-It is a combinatorial puzzle game played on an undirected colored graph $G=(V,E)$. Given a pivot element $p \in V$, the goal of the game is to turn the graph monochromatic with the smallest possible number of moves. A move consists of choosing a color and recoloring the monochromatic component of $G$ that contains $p$ with the chosen color. So, the game can be thought of as a series of flood-filling operations, where the monochromatic component that contains $p$ is considered to be the set of flooded vertices. 
The game is often played on a square $n\times n$ board where the top left cell is the pivot element, and the player is asked to solve it within a given number of moves, see \autoref{figure:game_introduction}.

\begin{figure}[ht]
    \centering
    \resizebox{\textwidth}{!}{
        	\begin{tikzpicture} [every node/.style={fill opacity=0.5, text opacity=100},
        proces/.style={outer sep=0pt,
            text width=0.3cm,
            align=center, anchor=center},
        square/.style={regular polygon,regular polygon sides=4},
		arrow/.style={-latex, shorten >=1ex, shorten <=1ex},
		graphnode/.style={fill opacity=0.5, text opacity=100},
		innernode/.style={square, proces, outer sep=0pt , draw, draw opacity=1, fill opacity=1, text opacity=0, scale=1},
		rednode/.style={innernode, draw=red, fill=red, node contents={1}},
		cyannode/.style={innernode,draw=cyan, fill=cyan, node contents={2}},
		greennode/.style={innernode,draw=green, fill=green, node contents={3}},
		yellownode/.style={innernode,draw=yellow, fill=yellow, node contents={4}},
		edgenode/.style={draw opacity=0, scale=0.5, midway, above=2pt, midway},
		rededgenode/.style={rednode, edgenode},
		cyanedgenode/.style={cyannode, edgenode},
		greenedgenode/.style={greennode, edgenode},
		yellowedgenode/.style={yellownode, edgenode}]
		\node (11)  at (0, 1){\begin{tikzpicture}[]
				\matrix[]{
						\node [rednode]  ; & \node [greennode] ; &  \node [cyannode] ; \\
						\node [greennode]  ; & \node [yellownode] ; &  \node [rednode] ; \\
						\node [greennode]  ; & \node [cyannode] ; &  \node [yellownode] ; \\
				};

		\end{tikzpicture}};
		\node  (12)  [right=of 11]  {\begin{tikzpicture}[]
				\matrix[]{
					\node [greennode]  ; & \node [greennode] ; &  \node [cyannode] ; \\
					\node [greennode]  ; & \node [yellownode] ; &  \node [rednode] ; \\
					\node [greennode]  ; & \node [cyannode] ; &  \node [yellownode] ; \\
				};
		\end{tikzpicture}};
		\node (13)  [right=of 12] {\begin{tikzpicture}[]
			\matrix[]{
				\node [cyannode]  ; & \node [cyannode] ; &  \node [cyannode] ; \\
				\node [cyannode]  ; & \node [yellownode] ; &  \node [rednode] ; \\
				\node [cyannode]  ; & \node [cyannode] ; &  \node [yellownode] ; \\
			};
		\end{tikzpicture}};
		\node (14)  [right=of 13]  {    \begin{tikzpicture}[]
			\matrix[]{
				\node [rednode]  ; & \node [rednode] ; &  \node [rednode] ; \\
				\node [rednode]  ; & \node [yellownode] ; &  \node [rednode] ; \\
				\node [rednode]  ; & \node [rednode] ; &  \node [yellownode] ; \\
			};
		\end{tikzpicture}};
		\node (15)  [right=of 14]  {    \begin{tikzpicture}[]
			\matrix[]{
				\node [yellownode]  ; & \node [yellownode] ; &  \node [yellownode] ; \\
				\node [yellownode]  ; & \node [yellownode] ; &  \node [yellownode] ; \\
				\node [yellownode]  ; & \node [yellownode] ; &  \node [yellownode] ; \\
			};
		\end{tikzpicture}};

		\draw [arrow]  (11) -- (12)  node [greenedgenode];
		\draw [arrow]  (12) -- (13) node [cyanedgenode];
		\draw [arrow]  (13) -- (14) node [rededgenode];
		\draw [arrow]  (14) -- (15) node [yellowedgenode];

	\end{tikzpicture}
    }

    \caption{A $3 \times 3$ Flood-It board on which a sequence of $4$ moves is played that solves the board.}
    \label{figure:game_introduction}
\end{figure}
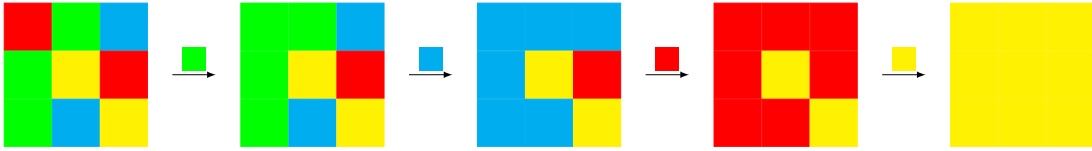

Finding \emph{some} solution to a Flood-It instance is trivial: Choosing any 'playable' color, that is, a color of a cell that is adjacent to the flooded region, increases the number of flooded cells. Playing an arbitrary sequence of playable colors will result in a completely flooded graph in a number of moves that is less than the number of vertices of the graph. 

In this paper, we consider the following two problems:

\begin{itemize}
    \item \textbf{Decision Problem:} Given a Flood-It instance and an integer $T$, decide whether there is a solution in at most $T$ moves (and provide such a solution).
    \item \textbf{Minimization Problem:} Given a Flood-It instance, determine the minimum number of moves required to solve it (and provide such a solution).
\end{itemize}

Finding a \emph{shortest} solution is trivial for instances with just two colors. In this case, there is only one playable move at each step. But already for three colors, the Flood-It decision problem is in general NP-complete and the minimization problem is NP-hard in the size of the graph, see \cite{Lagoutte2014FloodingGO} which also discusses certain edge cases that are not NP-hard.

\subsection{Literature Review on Flood-It Solving Approaches}
A solving approach based on SAT encoding for the decision problem was suggested in \cite{vanstiphoutSATEncodingsFloodIt2020}. For the minimization problem, an approach based on integer programming has been investigated in \cite{dasilvaEfficientApproachesFlooding2020}. To our knowledge, there are no publications that systematically compare different variants and approaches for either the decision or the minimization problem. Moreover, we are not aware of any scientific publication discussing the use of A* search for Flood-It. While there is an implementation available on Github, \cite{puchert_floodit}, an evaluation of its performance and comparison to other solving approaches is missing. That implementation does not exploit equivalences of game states for pruning, but uses a heuristic that was found to be efficient in our analysis.

\subsection{Structure of the Paper}
\autoref{sec:game_states} models the Flood-It game formally to derive properties that are important for the proposed optimizations of A* search in \autoref{sec:astar}. The two sections are interconnected in the following way: A classification of the vertices in \autoref{sec:game_states|vertex_categories} is introduced to derive an equivalence of game states in \autoref{sec:game_states|equiv}. This equivalence is used for pruning in A* search, see \autoref{sec:astar|equivalent_game_states}. The virtue of our equivalence is that it even allows identifying some game states that differ in the sequence of moves that lead to minimal solutions. Identifying large equivalence classes of game states results in stronger pruning and therefore more efficient performance in A* search.

Two different lower bounds on the minimum solution length are discussed in \autoref{sec:game_states|lower_bounds} to be subsequently used as heuristics in A* search, see \autoref{sec:astar|lower_bounds}.

Completely different approaches are presented  in \autoref{sec:solvers} based on declarative solvers, specifically using Boolean Satisfiability (SAT), Integer Linear Programming (ILP) and Constraint Programming (CP). For these formulations of Flood-It, the formalization of \autoref{sec:game_states} is not required. 

The generation of puzzle instances and some details on the implementation of the solving approaches are presented in \autoref{sec:experiments}, followed by the results in \autoref{sec:results}.

\autoref{sec:discussion} and \ref{sec:conclusion} provide a discussion and conclusion.

\section{Properties of Game States}
\label{sec:game_states}
We formalize the game Flood-It and derive properties of its game states, in particular an equivalence relation and lower bounds on the minimum solution. These findings will be exploited in \autoref{sec:astar} to tailor A* search for finding minimal Flood-It solutions more efficiently.

\subsection{Notation and Definitions}
\label{sec:game_states|definitions}
For a graph $G=(V,E)$ we call a mapping (\emph{coloring}) $c_G:V\rightarrow \{0,\dots,C-1\}$, where $C$ is the number of colors, a \emph{proper coloring} if it assigns different colors to adjacent vertices. For every graph $G$ with coloring $c_G$ that has adjacent vertices with the same color, an equivalent subgraph $G'$ with a proper coloring $c_{G'}$ can be constructed such that every sequence of moves that solves $G$ also solves $G'$ and vice versa \cite{vu_extremal_2019}.
Therefore, we will from here on assume all colorings to be proper to ease the theoretical handling. 

For a graph $G=(V,E)$ and a set $A$ of vertices, we define
\begin{itemize}
    \item the \emph{vertex set} of $G$ as $V(G):=V$,
    \item the \emph{edge set} of $G$ as $E(G):= E$,
    \item the \emph{closed neighborhood} of $A$ as $N_G[A] = \{ v\in V \;|\;\exists w\in A\; \{v,w\}\in E \}$, and
    \item the \emph{open neighborhood} of $A$ as $N_G(A) = N_G[A] \,\backslash\, A$.
\end{itemize}
A \emph{path} in a graph $G$ is a sequence of distinct vertices $(v_1, ...., v_n)$ such that $\forall i \in \{1, ..., n-1\}(\{v_i, v_{i+1}\} \in E(G))$. The length of a path is $n-1$. We say \emph{uv-path} for a path whose endpoints are two distinct vertices $u$ and $v$.

\subsection{Modeling of the Flood-It Game}
\label{sec:game_states|modelling}
If gameplay was modeled by recoloring the monochromatic component of the pivot element, it is possible that two different series of moves result in a coloring that differs only in the color of the monochromatic component. While the resulting configurations are not equal, they are equivalent with respect to the number of moves required to solve them. We define a different representation of the game so that two such series of moves will result in the same game state, which helps to reduce the search space.

Rather than modeling the moves of the game by changing the coloring of the graph, we keep that fixed and keep track of the set of flooded vertices $F_G$ of the graph $G$. In the initial state, this set consists just of the pivot vertex $F_G=\{p\}$, and the game is solved as soon as it encompasses all vertices, i.e., $F_G=V$.

\begin{definition}[Game State and Moves]
    A Flood-It game state for a connected undirected graph $G$ is a tuple $S = (G, c_G, F_G)$ where $c_G$ is a proper coloring and $F_G \subseteq V(G)$ is the set of flooded vertices. We may write $F_G^{0}$ if we explicitly refer to the initial flooding and $F_G^{i}$ for the flooding after the $i$'th move. A move $m_i$ with color $c$ from a state $S_{i-1}$ to $S_i$ is a transition $(G, c_G, F_G^{i-1}) \rightarrow^{m_i} (G, c_G, F_G^{i})$ such that $F_G^{i} = F_G^{i-1} \cup \{ v\in  N_G(F_G^{i-1}) : c_G(v) = c \}$. A game state $S$ is solved if and only if $F_G = V(G)$. We say a series of moves $(G, c_G, F_G^{0}) \rightarrow^{m_1} ... \rightarrow^{m_i} (G, c_G, F_G^{i})$ solves the initial game state if and only if $(G, c_G, F_G^{i})$ is solved. We refer to a color as solved if all vertices of that color are flooded.
\end{definition}

\begin{definition}
    Let S be a game state. The minimum number of moves required to solve S is denoted by:
    \begin{equation*}
        min_S
    \end{equation*}
\end{definition}

\subsection{Classification of Vertices}
\label{sec:game_states|vertex_categories}
Based on the definition of a game state and its underlying graph, this section defines and discusses subsets of its vertices that may have properties that are of interest for pruning algorithms and further allow a stronger statement about equivalence between game states with respect to shortest sequences of moves that solve them.

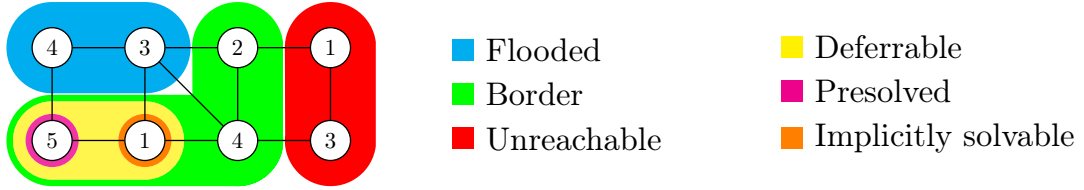
\begin{figure}[ht]
    \centering
    \resizebox{\textwidth}{!}{
            \newcommand{\convexpath}[2]{
        [   
        create hullnodes/.code={
            \global\edef\namelist{#1}
            \foreach [count=\counter] \nodename in \namelist {
                \global\edef\numberofnodes{\counter}
                \node at (\nodename) [draw=none,name=hullnode\counter] {};
            }
            \node at (hullnode\numberofnodes) [name=hullnode0,draw=none] {};
            \pgfmathtruncatemacro\lastnumber{\numberofnodes+1}
            \node at (hullnode1) [name=hullnode\lastnumber,draw=none] {};
        },
        create hullnodes
        ]
        ($(hullnode1)!#2!-90:(hullnode0)$)
        \foreach [
        evaluate=\currentnode as \previousnode using \currentnode-1,
        evaluate=\currentnode as \nextnode using \currentnode+1
        ] \currentnode in {1,...,\numberofnodes} {
            let
            \p1 = ($(hullnode\currentnode)!#2!-90:(hullnode\previousnode)$),
            \p2 = ($(hullnode\currentnode)!#2!90:(hullnode\nextnode)$),
            \p3 = ($(\p1) - (hullnode\currentnode)$),
            \n1 = {atan2(\y3,\x3)},
            \p4 = ($(\p2) - (hullnode\currentnode)$),
            \n2 = {atan2(\y4,\x4)},
            \n{delta} = {-Mod(\n1-\n2,360)}
            in 
            {-- (\p1) arc[start angle=\n1, delta angle=\n{delta}, radius=#2] -- (\p2)}
        }
        -- cycle
    }
    
    \begin{tikzpicture} [every node/.style={fill opacity=0.5, text opacity=100},
        arrow/.style={-latex, shorten >=1ex, shorten <=1ex},
        graphnode/.style={fill opacity=0.5, text opacity=100},
        innernode/.style={draw, draw opacity=1, circle, fill opacity=1, text opacity=100, scale=0.7},
        rednode/.style={innernode, fill=white, node contents={1}},
        cyannode/.style={innernode, fill=white, node contents={2}},
        greennode/.style={innernode, fill=white, node contents={3}},
        yellownode/.style={innernode, fill=white, node contents={4}},
        magentanode/.style={innernode, fill=white, node contents={5}},
        orangenode/.style={innernode, fill=white, node contents={6}},
        edgenode/.style={draw opacity=0, scale=0.8, midway, above},
        rededgenode/.style={rednode, edgenode},
        orangeedgenode/.style={orangenode, edgenode},
        cyanedgenode/.style={cyannode, edgenode},
        greenedgenode/.style={greennode, edgenode},
        blueedgenode/.style={yellownode, edgenode},
        magentaedgenode/.style={magentanode, edgenode}]

        \node (22)  at (0, 0)
        {
            \begin{tikzpicture}[]
                \node (1) at (1,2) [yellownode];
                \node (2) at (1,1) [magentanode];
                
                \node (3) at (2,2) [greennode];
                \node (4) at (2,1) [rednode];
                
                \node (5) at (3,2) [cyannode];
                \node (6) at (3,1) [yellownode];
                
                \node (7) at (4,2) [rednode];
                \node (8) at (4,1) [greennode];

                \draw (1) -- (2);
                \draw (3) -- (4);
                \draw (5) -- (6);
                
                \draw (1) -- (3) -- (5) -- (7) -- (8);
                \draw (2) -- (4) -- (6);
                \draw (3) -- (6);
                \draw (6) -- (8);
                
                \begin{scope}[on background layer]
                    
                    \fill[cyan,opacity=1] \convexpath{1,3}{14pt};
                    \fill[red,opacity=1] \convexpath{7,8}{14pt};
                    \fill[green,opacity=1] \convexpath{2, 4, 6, 5, 6}{14pt};
                    

                    \fill[white,opacity=1] \convexpath{2, 4}{12 pt};
                    \fill[yellow,opacity=0.7] \convexpath{2, 4}{12 pt};
                    
                    \node[fill=white,opacity=1, shape=circle, scale=0.6, fit=(2)] {};
                    \node[fill=magenta,opacity=0.8, shape=circle, scale=0.6, fit=(2)] {};
                    
                    \node[fill=white,opacity=1, shape=circle, scale=0.6, fit=(4)] {};
                    \node[fill=orange,opacity=1, shape=circle, scale=0.6, fit=(4)] {};
                \end{scope}
            \end{tikzpicture}
        };
        
        \matrix [draw, opacity=0] at (4,0) {
            \node [fill=cyan, fill opacity=1, label=right:Flooded] {}; \\
            \node [fill=green, fill opacity=1, label=right:Border] {}; \\
            \node [fill=red, fill opacity=1, label=right:Unreachable] {}; \\
        };
            \matrix [draw, opacity=0] at (8,0) {
            \node [fill=yellow, fill opacity=1, label=right:Deferrable] {}; \\
            \node [fill=magenta, fill opacity=1, label=right:Presolved] {}; \\
            \node [fill=orange, fill opacity=1, label=right:Implicitly solvable] {}; \\
        };
%
%
%
%
%
%
%
%
%
%
%
%
%

    \end{tikzpicture}
    }
    \caption{
        A visualization of how the vertices of a game state can be classified. The upper left vertex is the pivot element.
        The numbers of the vertices represent the colors of the Flood-It game. The background colors indicate to which class each vertex belongs. Vertices in the \emph{border} constitute the next possible moves. \emph{Unreachable} vertices cannot be flooded within one move. \emph{Deferrable} vertices could be flooded in the next move, but a shortest solution cannot be missed by not flooding them right away. The color of a \emph{presolved} vertex is completely solved by one move. That color can be played at any time without changing the solution length. For \emph{implicitly solvable} vertices, it is beneficial to wait playing them, as that provides the opportunity to flood more vertices of that color at the same move.
    }
    \label{figure:border}
\end{figure}

\begin{definition}[flooded]
    \label{def:flooded}
    For a game state $S = (G, c_G, F_G)$ we write $F_S$ to denote the set of flooded vertices:
    \begin{equation*}
        F_S := F_G
    \end{equation*}
\end{definition}

\begin{definition}[border]
    \label{def:border}
    For a game state $S = (G, c_G, F_G)$ the \emph{border} is the set of vertices adjacent to the set of flooded vertices. Colors occurring in the border are candidates for the next move. Formally:
    \begin{equation*}
        B_S := N_G(F_S)
    \end{equation*}
\end{definition}

\begin{definition}[unreachable]
    \label{def:unreachable}
    For a game state $S = (G, c_G, F_G)$, we call a vertex \emph{unreachable} if it takes more than one move to flood it. The set of unreachable vertices contains all non-flooded vertices of $S$, that are not part of the border:
    \begin{equation*}
        U_S := V(G) \setminus (F_S \cup B_S)
    \end{equation*}
\end{definition}

\begin{definition}[deferrable]
    \label{def:deferrable}
    For a game state $S = (G, c_G, F_G)$, a vertex is \emph{deferrable} if it can be flooded by making one move and is not adjacent to an unreachable vertex. Flooding these vertices does not open any new possibilities and can therefore safely be deferred to a later move. Formally
    \begin{equation*}
        D_S := B_S \setminus N_G(U_S)
    \end{equation*}
\end{definition}

\noindent
The set of deferrable vertices can further be divided into two disjoint sets of vertices:

\begin{definition}[presolved]
    \label{def:presolved}
    For a game state $S = (G, c_G, F_G)$, a deferrable vertex $v \in D_S$ is called \emph{presolved} if there is no vertex left in the closed neighborhood of the set of unreachable vertices that has the same color.
    \begin{equation*}
        P_S := \{ v \in D_S : \neg\exists w \in  N[U_S]\; c_G(v) = c_G(w) \}
    \end{equation*}
\end{definition}

\noindent
The time point of when to flood a presolved vertex is not relevant:

\begin{corollary}
\label{cor:presolved}
    The color of a presolved vertex can (and must) be solved by doing exactly one move that does not have any side effects because it has no neighbor that is not already in the border (i.e., adjacent to the set of flooded vertices). Thus, such a move can be played anytime without compromising the minimality of the solution.
\end{corollary}

\begin{definition}[implicitly solvable]
    \label{def:implicitly_solvable}
    For a game state $S = (G, c_G, F_G)$ we call a vertex \emph{implicitly solvable}, if it is deferrable and not presolved. These vertices can be ignored, as they will eventually be flooded together with other vertices, see \autoref{cor:deferrable}. The set of implicitly solvable vertices is defined as:
    \begin{equation*}
        I_S = D_S \setminus P_S = \{ v \in D_S : \exists w \in  N_G[U_S]\; c_G(v) = c_G(w) \}
    \end{equation*}
\end{definition}

\noindent
Any sequence of moves that floods the closed neighborhood of the unreachable vertices will also flood all implicitly solvable vertices. There are two cases:
\begin{itemize}
    \item An implicitly solvable vertex is flooded by a move that at the same time floods additional vertices of the same color that are not implicitly solvable.

    \item An implicitly solvable vertex turns presolved before it gets solved.
\end{itemize}

\noindent
Accordingly, implicitly solvable vertices do not require attention. Playing their color leads to inferior solutions unless there is a non-deferrable vertex of the same color.

In combination with \autoref{cor:presolved} we can conclude the following strategy for deferrable vertices:

\begin{corollary}
\label{cor:deferrable}
Deferrable vertices do not have to be considered for minimal move sequences. Their color should only be played if it also occurs in the closed neighborhood of the unreachable vertices or if there are only deferrable vertices left. The latter case applies if all deferrable vertices are presolved ($D_S = P_S$).
\end{corollary}

\subsection{Equivalent Game States}
\label{sec:game_states|equiv}
Identifying `equivalent' game states can help to make A* search more efficient. If it is known that several game states have optimal solutions of the same length, only one of them needs to be investigated further, while the others can be pruned from searching. However, in general, it is not possible to infer that two game states have the same minimum solution length without performing a deep search on both, which is precisely what one wants to avoid. Often it is possible to derive conditions under which two game states have a common optimal solution and use these for equivalence. For Flood-It, we managed to derive an equivalence that is a bit more general, as it allows identifying game states having the same minimum solution length, although they do not share a minimum move sequence.

For illustration, an example is shown in \autoref{figure:presolved}: Two different series of moves result in different game states for which the union of the set of flooded and deferrable vertices is the same. Even if they do not share a common shortest solution, these states can be solved with the same number of moves if the count of distinct colors in their respective sets of presolved vertices is the same. We prove that this holds true for any two states with these properties in \autoref{theorem:game_states_equal_min}.

This insight on equivalent game states can be applied to narrow the search space, as will be discussed in \autoref{sec:astar|equivalent_game_states} on A*.

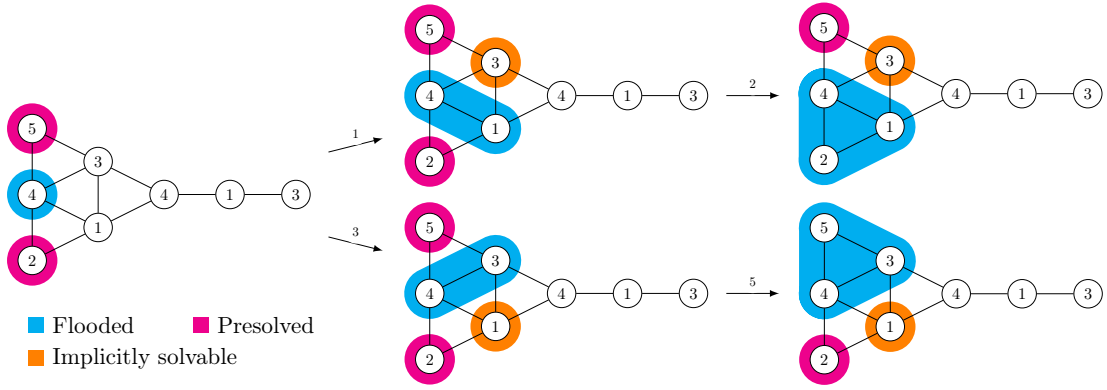
\begin{figure}[ht]
    \resizebox{\textwidth}{!}{
            \newcommand{\convexpath}[2]{
        [   
        create hullnodes/.code={
            \global\edef\namelist{#1}
            \foreach [count=\counter] \nodename in \namelist {
                \global\edef\numberofnodes{\counter}
                \node at (\nodename) [draw=none,name=hullnode\counter] {};
            }
            \node at (hullnode\numberofnodes) [name=hullnode0,draw=none] {};
            \pgfmathtruncatemacro\lastnumber{\numberofnodes+1}
            \node at (hullnode1) [name=hullnode\lastnumber,draw=none] {};
        },
        create hullnodes
        ]
        ($(hullnode1)!#2!-90:(hullnode0)$)
        \foreach [
        evaluate=\currentnode as \previousnode using \currentnode-1,
        evaluate=\currentnode as \nextnode using \currentnode+1
        ] \currentnode in {1,...,\numberofnodes} {
            let
            \p1 = ($(hullnode\currentnode)!#2!-90:(hullnode\previousnode)$),
            \p2 = ($(hullnode\currentnode)!#2!90:(hullnode\nextnode)$),
            \p3 = ($(\p1) - (hullnode\currentnode)$),
            \n1 = {atan2(\y3,\x3)},
            \p4 = ($(\p2) - (hullnode\currentnode)$),
            \n2 = {atan2(\y4,\x4)},
            \n{delta} = {-Mod(\n1-\n2,360)}
            in 
            {-- (\p1) arc[start angle=\n1, delta angle=\n{delta}, radius=#2] -- (\p2)}
        }
        -- cycle
    }
    
    \begin{tikzpicture} [every node/.style={fill opacity=0.5, text opacity=100},
        arrow/.style={-latex, shorten >=1ex, shorten <=1ex},
        graphnode/.style={fill opacity=0.5, text opacity=100},
        innernode/.style={draw, draw opacity=100, circle, fill opacity=1, text opacity=100, scale=0.7},
        rednode/.style={innernode, fill=white, node contents={1}},
        cyannode/.style={innernode, fill=white, node contents={2}},
        greennode/.style={innernode, fill=white, node contents={3}},
        yellownode/.style={innernode, fill=white, node contents={4}},
        magentanode/.style={innernode, fill=white, node contents={5}},
        orangenode/.style={innernode, fill=white, node contents={6}},
        edgenode/.style={draw opacity=0, scale=0.8, midway, above},
        rededgenode/.style={rednode, edgenode},
        orangeedgenode/.style={orangenode, edgenode},
        cyanedgenode/.style={cyannode, edgenode},
        greenedgenode/.style={greennode, edgenode},
        yellowedgenode/.style={yellownode, edgenode},
        magentaedgenode/.style={magentanode, edgenode}]
        
        \node (11)  at (-1, 1)
        {
            \begin{tikzpicture}[]
                \node (1) at (1,1) [yellownode];
                \node (2) at (1,2) [magentanode];
                \node (3) at (1,0) [cyannode];
                \node (4) at (2,0.5) [rednode];
                \node (5) at (2,1.5) [greennode];
                \node (6) at (3,1) [yellownode];
                \node (7) at (4,1) [rednode];
                \node (8) at (5,1) [greennode];
                \draw (1) -- (4) -- (6) -- (7) -- (8);
                \draw (1) -- (5);
                \draw (1) -- (3);
                \draw (1) -- (2);
                \draw (2) -- (5);
                \draw (3) -- (4);
                \draw (4) -- (5);
                \draw (5) -- (6);
                
                \begin{scope}[on background layer]
                    \node[fill=cyan,opacity=1.0, shape=circle, scale=0.8, fit=(1)] {};
                    \node[fill=magenta,opacity=1.0, shape=circle, scale=0.8, fit=(2)] {};
                    \node[fill=magenta,opacity=1.0, shape=circle, scale=0.8, fit=(3)] {};
                \end{scope}
            \end{tikzpicture}
        };
        \node (12)  at (5, 2.5)
        { 
            \begin{tikzpicture}[]
                \node (1) at (1,1) [yellownode];
                \node (2) at (1,2) [magentanode];
                \node (3) at (1,0) [cyannode];
                \node (4) at (2,0.5) [rednode];
                \node (5) at (2,1.5) [greennode];
                \node (6) at (3,1) [yellownode];
                \node (7) at (4,1) [rednode];
                \node (8) at (5,1) [greennode];
                \draw (1) -- (4) -- (6) -- (7) -- (8);
                \draw (1) -- (5);
                \draw (1) -- (3);
                \draw (1) -- (2);
                \draw (2) -- (5);
                \draw (3) -- (4);
                \draw (4) -- (5);
                \draw (5) -- (6);
                
                \begin{scope}[on background layer]
                    \fill[cyan,opacity=1.0] \convexpath{1,4}{11pt};
                    \node[fill=magenta,opacity=1.0, shape=circle, scale=0.8, fit=(2)] {};
                    \node[fill=magenta,opacity=1.0, shape=circle, scale=0.8, fit=(3)] {};
                    \node[fill=orange,opacity=1.0, shape=circle, scale=0.8, fit=(5)] {};
                \end{scope}
            \end{tikzpicture}
        };
        \node (13)  at (11, 2.5)
        { 
            \begin{tikzpicture}[]
                \node (1) at (1,1) [yellownode];
                \node (2) at (1,2) [magentanode];
                \node (3) at (1,0) [cyannode];
                \node (4) at (2,0.5) [rednode];
                \node (5) at (2,1.5) [greennode];
                \node (6) at (3,1) [yellownode];
                \node (7) at (4,1) [rednode];
                \node (8) at (5,1) [greennode];
                \draw (1) -- (4) -- (6) -- (7) -- (8);
                \draw (1) -- (5);
                \draw (1) -- (3);
                \draw (1) -- (2);
                \draw (2) -- (5);
                \draw (3) -- (4);
                \draw (4) -- (5);
                \draw (5) -- (6);
                
                \begin{scope}[on background layer]
                    \fill[cyan,opacity=1.0] \convexpath{1,4,3,1}{11pt};
                    \node[fill=magenta,opacity=1.0, shape=circle, scale=0.8, fit=(2)] {};
                    \node[fill=orange,opacity=1.0, shape=circle, scale=0.8, fit=(5)] {};
                \end{scope}
            \end{tikzpicture}
        };

        \node (22)  at (5, -0.5)
        { 
            \begin{tikzpicture}[]
                \node (1) at (1,1) [yellownode];
                \node (2) at (1,2) [magentanode];
                \node (3) at (1,0) [cyannode];
                \node (4) at (2,0.5) [rednode];
                \node (5) at (2,1.5) [greennode];
                \node (6) at (3,1) [yellownode];
                \node (7) at (4,1) [rednode];
                \node (8) at (5,1) [greennode];
                \draw (1) -- (4) -- (6) -- (7) -- (8);
                \draw (1) -- (5);
                \draw (1) -- (3);
                \draw (1) -- (2);
                \draw (2) -- (5);
                \draw (3) -- (4);
                \draw (4) -- (5);
                \draw (5) -- (6);
                
                \begin{scope}[on background layer]
                    \fill[cyan,opacity=1.0] \convexpath{1,5}{11pt};
                    \node[fill=magenta,opacity=1.0, shape=circle, scale=0.8, fit=(2)] {};
                    \node[fill=magenta,opacity=1.0, shape=circle, scale=0.8, fit=(3)] {};
                    \node[fill=orange,opacity=1.0, shape=circle, scale=0.8, fit=(4)] {};
                \end{scope}
            \end{tikzpicture}
        };
        \node (23)  at (11, -0.5)
        { 
            \begin{tikzpicture}[]
                \node (1) at (1,1) [yellownode];
                \node (2) at (1,2) [magentanode];
                \node (3) at (1,0) [cyannode];
                \node (4) at (2,0.5) [rednode];
                \node (5) at (2,1.5) [greennode];
                \node (6) at (3,1) [yellownode];
                \node (7) at (4,1) [rednode];
                \node (8) at (5,1) [greennode];
                \draw (1) -- (4) -- (6) -- (7) -- (8);
                \draw (1) -- (5);
                \draw (1) -- (3);
                \draw (1) -- (2);
                \draw (2) -- (5);
                \draw (3) -- (4);
                \draw (4) -- (5);
                \draw (5) -- (6);
                
                \begin{scope}[on background layer]
                    \fill[cyan,opacity=1.0] \convexpath{1,2,5,1}{11pt};
                    \node[fill=magenta,opacity=1.0, shape=circle, scale=0.8, fit=(3)] {};
                    \node[fill=orange,opacity=1.0, shape=circle, scale=0.8, fit=(4)] {};
                \end{scope}
            \end{tikzpicture}
        };

        \draw [arrow]  (11) -- (12)  node [rededgenode];
        \draw [arrow]  (12) -- (13) node [cyanedgenode];
        
        \draw [arrow]  (11) -- (22)  node [greenedgenode];
        \draw [arrow]  (22) -- (23) node [magentaedgenode];

        \node (23)  at (-0.75, -1.25)
        { %
            \begin{tikzpicture}[every label/.style={text opacity=100}]
                 
                 \node (2) at (0,0.5) [fill=cyan, fill opacity=1.0, label={[black, text opacity=1.0]right:Flooded}] {};
                 \node (3) at (2.5,0.5) [fill=magenta, fill opacity=1, label={[black, text opacity=1.0]right:Presolved}] {};
                 \node (1) at (0,0) [fill=orange, fill opacity=1, label={[black, text opacity=1.0]right:Implicitly solvable}] {};
            \end{tikzpicture}
        };
    \end{tikzpicture}
    }
    \caption{
        Two different series of moves (here $(1,2)$ and $(3,5)$) can result in different game states (on the right) that may or may not have a common shortest sequence of moves that solves them, while the minimum number of required moves to solve both states is the same. After the first move, both resulting game states share a shortest sequence of moves that solves them $(4, 1, 3, 2, 5)$. After the second move the two resulting states do not share such a sequence, but still only 4 moves are required to solve them. They can both be solved by playing $(4, 1, 3)$ followed by the color of the remaining presolved vertex.}
    \label{figure:presolved}
\end{figure}

\begin{theorem}
    \label{theorem:game_states_equal_min}
    For two game states $S_1 = (G, c_G, F_{S_1})$ and $S_2 = (G, c_G, F_{S_2})$  with corresponding sets of presolved vertices $P_{S_1}$ and  $P_{S_2}$, the minimum number of moves required to solve these states is equal if:
    \begin{align*}
        F_{S_1} \cup D_{S_1} = F_{S_2} \cup D_{S_2} \\
        |c_G(P_{S_1})| = |c_G(P_{S_2})|
    \end{align*}
\end{theorem}

\begin{proof}
We write $A\cupdot B$ to indicate the union is formed by disjoint sets $A$ and $B$. For any game state $S=(G, c_G, F_S)$, we get by definition of the border and unreachable vertices that $(F_S \cupdot B_S) \cupdot U_S = V(G)$.
Further $N_G(U_S) \cupdot D_S =  B_S$ thus $(F_S\cupdot D_S)\cupdot N_G[U_S] = V(G)$.

Now, let $S_1$ and $S_2$ be given as in the statement of the theorem. Then $N_G[U_{S_1}] = N_G[U_{S_2}]$, and because both game states have the same graph and coloring, a shortest sequence of moves $m$ that floods $N_G[U_{S_1}]$ is also minimal for $N_G[U_{S_2}]$ and vice versa.
Such a sequence will also flood all respective implicitly solvable vertices of $S_1$ and $S_2$.
Because it holds true that $D_S = I_S \cupdot P_S$ for any game state $S$, only their respective presolved vertices will not be flooded after playing $m$ on $S_1$ and $S_2$. Because the presolved vertices can be flooded by playing each of their colors once, and these moves can be played anytime without side effects, see \autoref{cor:presolved}, we get that $min_{S_1} = |m| + |c_G(P_{S_1})| = |m| + |c_G(P_{S_2})| = min_{S_2}$.
\end{proof}

\subsection{Lower Bounds}
\label{sec:game_states|lower_bounds}
In this section, we will obtain two lower bounds for the number of moves required to solve a game state. To determine a lower bound, we relax the rules of the game by allowing more powerful moves and use the number of moves required in the relaxed game as a lower bound for the actual game. Lower bounds can be used as heuristics in A* search, see \autoref{sec:astar}.

Note that a simple lower bound can be defined by the number of unsolved colors ($LBC$ for \emph{Lower Bound by unsolved Colors}). Anticipating that this leaves much room for improvement, we will present two more sophisticated bounds and use LBC as the punchline.

Our first bound will be based on a level structure of a game state's graph, which we will define first. As of \cite{diaz_survey_2002} a \textit{level structure} of a (general) graph is a partition $L_0, ..., L_s$ of its vertex set such that the endpoints of every edge in the graph are either in the same level $L_i$ or in two consecutive levels $L_i$ and $L_{i+1}$. We will define the level structure for a Flood-It game state accordingly.

\begin{definition}[Level Structure of a Game State]
    Let $S = (G, c_G, F_G)$ be a game state on a graph $G$ with proper coloring. The corresponding level structure is defined such that:
    \begin{align*}
        L_0 &= F_S \\
        L_i &= N_G(L_{i-1}) \setminus (\bigcup_{0 \leq j < i}{L_j})
    \end{align*}
    
    When referring to a level structure $L_1, ..., L_s$, we will (unless stated otherwise) assume $L_s$ to be the maximum level such that $L_s \neq \emptyset \text{ and } L_{s+1} = \emptyset $.
    \label{def:levelstructure}
\end{definition}

Note that when interpreting $\bigcup_{0 \leq j \leq i}{L_j}$ as the set of flooded vertices, $L_{i+1}$ is the border. The first lower bound we propose can be interpreted as a series of `combined' moves in which each move plays all colors of each level at once up to a specific level, followed by playing each color of the next level (border) separately.

\begin{definition}
    \label{def:lbls} \textbf{Lower Bound by Level Structure (LBL)}
    Let $S = (G, c_G, F_G)$ be a game state on a graph $G$. We define the lower bound $LBL(S)$ based on the level structure $L_0, \dots, L_s$ of $G$:
    \begin{equation*}
        LBL(S) := \max_{i \in \{1, ..., s\}} i + |c_G(L_i)| - 1
    \end{equation*}
\end{definition}

\begin{lemma}
    \label{lem:lbfurthest}
    The function $LBL(S)$ is indeed a lower bound for the minimum number of moves to solve game state $S = (G, c_G, F_G)$:
    \begin{equation*}
        LBL(S) \leq \min{}_S
    \end{equation*}
\end{lemma}

\begin{proof}
    Let $L_0, \dots, L_s$ be the level structure of $G$ and $0 < i \leq s$. By construction of the level structure, it holds true that for every vertex $v \in L_i$ there exists a vertex $u \in L_0$ so that a shortest $uv$-path has length $i$. Because $c_G$ is a proper coloring, no two adjacent vertices of that path have the same color. Thus, it takes at least $i$ moves to flood $v$. Now let $m_1, ..., m_i$ be a series of moves that floods $v$. A move $m_{j}$ with $j<i$ can only flood vertices in layers $L_1, ..., L_j$. Therefore, these moves cannot flood any vertex in $L_i$. Playing $m_i$ reduces the number of unflooded colors in $L_i$ by one such that there are $k = |c_G(L_i)| - 1$ unflooded colors left in $L_i$. Thus, it would take additional $k$ moves to flood all vertices in $L_i$. And thus $min_S \geq \max_{i \in \{1, ..., s\}} i + |c_G(L_i)| - 1$
\end{proof}

A weakness of this lower bound $LBL$ is that it does not take into consideration that there might be colors which get \emph{solved} by a combined move. 
Such colors could be played and counted as separate moves instead of being merged into a combined move, which leads to a higher and therefore tighter lower bound.
An algorithm which takes such moves into consideration was implemented by Puchert \cite{puchert_floodit}. The pseudocode of \autoref{algo:lower_bounds_puchert} reflects the idea behind their implementation, which we call $LBP$ (for \emph{Lower Bound Puchert} in the sequel).

\begin{algorithm}
    \caption{LBP$(S)$}
    \begin{algorithmic}[1]
        \State $lb \gets 0$
        \While {$S$ not solved}
        \State $A \gets \{u \in B_S : \neg\exists v \in U_S (c_G(u) = c_G(v)) \}$  \Comment{vertices of safe moves}
        \If {$A \neq \emptyset$}
        \State $lb \gets lb + |c_G(A)|$     \Comment{add one for each safe move}
        \State $F_S \gets F_S \cup A$   \Comment{add safe moves to flooded vertices}
        \Else
        \State $lb \gets lb + 1$            \Comment{add one for the combined move}
        \State $F_S \gets F_S \cup B_S$ \Comment{flood the whole border}
        \EndIf
        \EndWhile
        \State \Return $lb$ 
    \end{algorithmic}
    \label{algo:lower_bounds_puchert}
\end{algorithm}

\autoref{bounds:puchert} shows a graph on which this algorithm returns a lower bound of $4$ which is tighter than the lower bound of $3$ that $LBL$ (\autoref{lem:lbfurthest}) provides.

\begin{figure}[ht]
   \centering
   	\newcommand{\convexpath}[2]{
		[   
		create hullnodes/.code={
			\global\edef\namelist{#1}
			\foreach [count=\counter] \nodename in \namelist {
				\global\edef\numberofnodes{\counter}
				\node at (\nodename) [draw=none,name=hullnode\counter] {};
			}
			\node at (hullnode\numberofnodes) [name=hullnode0,draw=none] {};
			\pgfmathtruncatemacro\lastnumber{\numberofnodes+1}
			\node at (hullnode1) [name=hullnode\lastnumber,draw=none] {};
		},
		create hullnodes
		]
		($(hullnode1)!#2!-90:(hullnode0)$)
		\foreach [
		evaluate=\currentnode as \previousnode using \currentnode-1,
		evaluate=\currentnode as \nextnode using \currentnode+1
		] \currentnode in {1,...,\numberofnodes} {
			let
			\p1 = ($(hullnode\currentnode)!#2!-90:(hullnode\previousnode)$),
			\p2 = ($(hullnode\currentnode)!#2!90:(hullnode\nextnode)$),
			\p3 = ($(\p1) - (hullnode\currentnode)$),
			\n1 = {atan2(\y3,\x3)},
			\p4 = ($(\p2) - (hullnode\currentnode)$),
			\n2 = {atan2(\y4,\x4)},
			\n{delta} = {-Mod(\n1-\n2,360)}
			in 
			{-- (\p1) arc[start angle=\n1, delta angle=\n{delta}, radius=#2] -- (\p2)}
		}
		-- cycle
	}
	
	\begin{tikzpicture} [every node/.style={fill opacity=0.5, text opacity=100},
		arrow/.style={-latex, shorten >=1ex, shorten <=1ex},
		graphnode/.style={fill opacity=0.5, text opacity=100},
		innernode/.style={draw, draw opacity=100, circle, fill opacity=1, text opacity=100, scale=0.7},
		rednode/.style={innernode, fill=red, node contents={1}},
		cyannode/.style={innernode, fill=cyan, node contents={2}},
		greennode/.style={innernode, fill=green, node contents={3}},
		yellownode/.style={innernode, fill=yellow, node contents={4}},
		magentanode/.style={innernode, fill=magenta, node contents={5}},
		orangenode/.style={innernode, fill=orange, node contents={6}},
		edgenode/.style={draw opacity=0, scale=0.8, midway, above},
		rededgenode/.style={rednode, edgenode},
		orangeedgenode/.style={orangenode, edgenode},
		cyanedgenode/.style={cyannode, edgenode},
		greenedgenode/.style={greennode, edgenode},
		blueedgenode/.style={yellownode, edgenode},
		magentaedgenode/.style={magentanode, edgenode}]
		
		\node (11)  at (-1, 1)
		{
			\begin{tikzpicture}[]
				\node (1) at (0,1) [greennode];
				\node (2) at (1,0) [rednode];
				\node (3) at (1,1) [cyannode];
				\node (4) at (1,2) [orangenode];
				\node (5) at (2,1) [greennode];
				\node (6) at (3,1) [cyannode];
				
				\draw (1) -- (2) -- (5) -- (6);
				\draw (1) -- (3) -- (5);
				\draw (1) -- (4) -- (5);
				\draw (2) -- (3) -- (4);
				
				\node (11) at (0,2.5)[opacity=1] {$L_0$};
				\node (11d) at (0,0)[opacity=0] {};
				
				\node (12) at (1,2.5)[opacity=1] {$L_1$};
				\node (12d) at (1,0)[opacity=0] {};
				
				\node (13) at (2,2.5)[opacity=1] {$L_2$};
				\node (13d) at (2,0)[opacity=0] {};
				
				\node (14) at (3,2.5)[opacity=1] {$L_3$};
				\node (14d) at (3,0)[opacity=0] {};

				\begin{scope}[on background layer]
					\fill[yellow, opacity=0.3] \convexpath{11, 11d}{11pt};
					\fill[yellow, opacity=0.3] \convexpath{12, 12d}{11pt};
					\fill[yellow, opacity=0.3] \convexpath{13, 13d}{11pt};
					\fill[yellow, opacity=0.3] \convexpath{14, 14d}{11pt};
				\end{scope}
				
			\end{tikzpicture}
		};

		%
	\end{tikzpicture}
    \caption{
        On this graph, the lower bound $LBL$, see \autoref{def:lbls} equals 3 (yielded by $i=1$ as well as by $i=3$) while the lower bound $LBP$ obtained by \autoref{algo:lower_bounds_puchert} equals 4. The latter can be seen in the following way. Play safe moves $6,1$. Then the border consists of the $2$ in the middle and the $3$ to its right. These are played as a combined move and finally playing $2$ as the fourth (pseudo) move.
        \label{bounds:puchert}}
\end{figure}
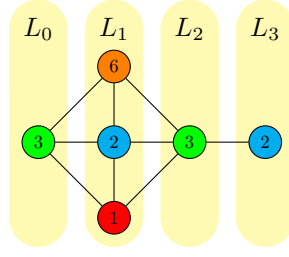

Before we can show that it indeed returns a lower bound, we will first introduce \emph{safe} moves, which Puchert refers to as \emph{optimal} in his code \cite{puchert_floodit}.

\begin{definition}[Safe Move]
    We call a move \emph{safe} if it solves the color that is played by this move.
    \label{def:optimal_move}
\end{definition}

Note, that playing the color of a presolved vertex is a safe move, but in general, the reverse is not correct. There can be safe moves that bring previously unreachable vertices into the border. Accordingly, the played vertex was not deferrable and, in particular, not presolved.

\begin{lemma}
    \label{lemma:optimal_move}
    Let $S$ be a game state with safe move $m_1$ and let $S_1$ be the resulting game state after playing $m_1$ on $S$ then:
    \begin{equation*}
        min_{S_1} = min_S - 1
    \end{equation*}
    Consequently, if there are several possible safe moves $m_1, ...., m_i$ which could be played on $S$, then after playing these moves $S \rightarrow^{m_1} ... \rightarrow^{m_i} S_i$ we obtain:
    \begin{equation*}
        min_{S_i} = min_{S} - i
    \end{equation*}
\end{lemma}

\begin{proof}
    Let $m_2,\dots,m_n$ be a shortest move sequence to solve $S_1$. Then move sequence $m_1,m_2,\dots,m_n$ solves $S$, thus $\min_S \le \min_{S_1}+1$. Conversely, let $p_1,\dots,p_k$ be a shortest sequence solving $S$. There must be exactly one move $p_j$ in the sequence that plays the color of the safe move $m_1$ by the definition of a safe move, i.e., $p_j=m_1$. Since $m_1$ is a safe move of state $S$, it can as well be played first. So, $m_1,p_1,\dots,p_{j-1},p_{j+1},\dots,p_k$ also solves $S$, entailing that $p_1,\dots,p_{j-1},p_{j+1},\dots,p_k$ solves $S_1$, proving $\min_{S_1} \le \min_S - 1$.
\end{proof}

For completeness, we show the following obvious property of monotonicity.

\begin{lemma}
    \label{lemma:monotony}
    If $S'=(G,c_G,F'_G)$ and $S=(G,c_G,F_G)$ are two game states of the same underlying colored graph $(G,c_G)$ with the flooded vertices of $S'$ being a subset of the flooded vertices of $S$ ($F'_G\subseteq F_G$), it holds true that
    \begin{equation*}
        \min{}_S \leq \min{}_{S'}
    \end{equation*}
\end{lemma}
\begin{proof}
    Note that, for a state sequence $S'_1,\dots,S'_n$ resulting from playing moves $m_1,\dots,m_n$ on $S'$, the flooded vertices $F_{S'_i}$ will be a subset of the flooded vertices $F_{S_i}$ with $S_1,\dots,S_n$ being the states resulting from playing the same move sequence on $S$. 
    
    For proving the claim, we consider an optimal move sequence $m_1,\dots,m_n$ for solving $S'$. We obtain $V(G) = F_{S'_n} \subseteq F_{S_n}$, so $F_{S_n} = V(G)$, showing that the same move sequence solves $S$.
\end{proof}

We can now introduce \autoref{algo:lower_bounds_puchert} and show that it returns a proper lower bound. It initializes the lower bound with $0$ and iteratively repeats these two steps until all vertices are flooded.

\begin{itemize}
    \item [(1)] \textit{If one or more safe moves are possible:} Flood all the vertices that could be flooded by playing all safe moves. Then add the number of solved colors to the lower bound.
    
    \item [(2)] \textit{Otherwise:} Flood the whole border and add one to the lower bound.
\end{itemize}
When $S$ is solved, the lower bound is returned.

Now, we can show that the function $LBP()$ is indeed a lower bound for the minimum number of moves:

\begin{lemma}
    \label{lem:lbpuchert}
    For any game state $S = (G, c_G, F_G)$, it holds that
    \begin{equation*}
        LBP(S) \leq \min{}_S
    \end{equation*}
\end{lemma}

\begin{proof}
Let $S_1$ be the game state at the beginning of an iteration, and let $S_2$ be the game state after an iteration. For considering case (1), let $i$ be the number of solved colors, which is added to the lower bound. \autoref{lemma:optimal_move} yields in this situation $min_{S_1} = i + min_{S_2}$. In case (2) the whole border is flooded and $1$ is added to the lower bound. The move in a shortest sequence for state $S_1$ can only flood a subset of the border. Call the resulting game state $S'_2$ we have $F_{S'_2} \subseteq F_{S_2}$. Using \autoref{lemma:monotony} we obtain $\min_{S_1} = 1 + \min_{S'_2} \geq 1 + min_{S_2}$. And thus, by induction $LBP(S) \leq \min_S$ is a lower bound.
\end{proof}

\section{A*}
\label{sec:astar}
A* is a graph path search algorithm between a start and a target node\footnote{%
To distinguish between the graph that represents the Flood-It game and the graph of game states for searching a solution, we use the term \emph{vertices} for the former and \emph{nodes} for the latter.
} 
that has been published by Hart et al. \cite{hart_formal_1968} in 1968. It performs a guided search that makes use of a heuristic to estimate the cost of a path from each discovered node to the goal. By doing so, the search space can be explored more target-oriented, which is done by selecting the most promising node to examine next in each iteration. A drawback to this approach is its space complexity because information about discovered nodes has to be stored in memory.

In the standard terminology that is used when talking about A*, $g(n)$ represents the cost of the best known path from start node $s$ to node $n$ (\emph{backward costs}). The function $h(n)$ estimates the cost of a path from $n$ to the target node $t$ (\emph{forward costs} or heuristic). To guide the search, the main loop of A* examines the node $n$ that has the lowest $f$-value, which is determined by:
\begin{align*}
    f(n) = g(n) + h(n)
\end{align*}

Typical implementations use some sort of priority queue to keep track of the set of nodes that are queued for expansion (\emph{open list}). The node with the lowest $f$-value has the highest priority.
Information about already expanded nodes may be stored in a \emph{closed list}, often implemented as a hash table, to avoid revisiting nodes that have already been visited.

When using A* search to find the shortest solution in a game, the nodes of the graph represent game states. Each node is connected by directed edges to nodes that correspond to game states reachable by performing one move from the parent. A (shortest) path from the root (initial game state) to a leaf node, i.e., a solved game state, represents a (shortest) solution of the game.

When a node is visited in A* search, it is expanded by generating child nodes for each possible move. In general, pruning strategies should be considered to improve search efficiency. In the Flood-It game, certain moves can be ignored without endangering missing a shortest solution. The next section discusses pruning strategies based on this idea.

\subsection{Pruning Strategies}
\label{sec:astar|branching_stragegies}
There might be moves, which can be omitted from consideration when expanding a node. For any game state $S$, at most the colors of the border $B_S$ must be considered as moves because a move that plays a different color would not flood any vertices and thus cannot result in an optimal solution.

As we have shown in \autoref{lemma:optimal_move}, safe moves (\autoref{def:optimal_move}) can be played at any time without leading to an inferior solution.

Additionally, as of \autoref{cor:deferrable} the colors of the deferrable vertices $D_S$ do not have to be considered as moves unless such a color also occurs in the closed neighborhood of the unreachable vertices or there are only deferrable vertices left. In the latter case $D_S = P_S$.

Thus, we can derive strategies to prune the search space and propose to use the following three rules to decide which moves should be considered for a game state $S = (G, c_G, F_G)$.
\begin{itemize}
    \item [1. \textbf{Op}] (Optional) If \textit{safe moves} can be played on $S$, one of these moves is played, while no alternative move is considered at this game state. The choice, if there is more than one safe move, is not relevant.
    
    \item [2. \textbf{Mc}] (Must check) If there are no \textit{safe moves} remaining or rule (1) is not applied, we branch into each non-deferrable color of the border $c \in c_G{(B_S \setminus D_S)}$, and ignore all other possible moves. Deferrable colors do not have to be checked.
    
    \item [3. \textbf{Ps}] (Presolved) If $B_S \setminus D_S = \emptyset$ there are only \textit{presolved} vertices left. In that case, the order in which their colors are played is not relevant, and therefore only one color $c \in c_G{(B_S)}$ needs to be considered for branching.
\end{itemize}

Because playing the color of a presolved vertex is by definition a safe move, these rules effectively result in two pruning strategies: [Op, Mc] and [Mc, Ps]. We would expect these to be more efficient compared to branching into all remaining colors of the unflooded vertices or border.

\subsection{Lower Bounds as Heuristics}
\label{sec:astar|lower_bounds}
The heuristic function $h(n)$ plays a key role for A* search. It must be admissible to guarantee that an optimal solution is found, which means that $h(n) \leq min_S$ holds for any node $n$ representing game state $S$. This is the case for the two lower bounds $LBL$ and $LBP$ presented in \autoref{sec:game_states|lower_bounds} according to \autoref{lem:lbfurthest} and \ref{lem:lbpuchert}, thus they can be used as heuristic for A*.

In view of efficiency, it is desirable that heuristics are consistent. A simple example shows that lower bound $LBL$ of \autoref{def:lbls} is not consistent, see \autoref{app:lbl_not_consistent} in the appendix. The second lower bound, $LBP$, gave consistent estimates across all checked Flood-It instances, while we do not know whether its consistency holds in general. Anyway, both lower bounds may be used as heuristic to determine optimal solutions with A* due to their admissibility. Lacking consistency, the search procedure is, in general, not greedy, which means that the same nodes might require multiple visits, which can hamper efficiency.

\subsection{Equivalent Game States }
\label{sec:astar|equivalent_game_states}
In a straightforward implementation of A*, each node of the search space represents a game state. And each of these nodes might potentially be expanded.
If we, however, have the guarantee, that two different game states of the search space can be solved with the same minimum number of moves, only one of them would need to be expanded.
If they have the same $g$-value any of them could be chosen, otherwise only the one with the lower $g$-value must be expanded as it is guaranteed to lead to a shorter overall solution.
So we can, based on the equivalence introduced in \autoref{theorem:game_states_equal_min}, treat nodes of the search space as equivalence classes such that each class contains the game states for which we can give this guarantee.
And at any time, the game state with the currently lowest known $g$-value will be used as the representative element for its corresponding equivalence class, while all other game states of the class may be discarded.
This should result in a smaller or, in the worst case, equal-sized search space, open and closed lists compared to a default implementation, and thus lead to a potentially better performance if implemented efficiently.

\subsection{Tie Breaking Strategies}
\label{sec:astar|tie_breaking}
The $f$-values are integers corresponding to estimated solution lengths. These are comparably small compared to the number of game states in the search queue (open list). Therefore, there will typically be many nodes that have the same $f$-value in the open list, which gives a choice when A* selects the next node for exploration based on the minimum $f$-value.
According to Hart et al.\@ \cite{hart_formal_1968} such ties can be resolved arbitrarily. But according to Asai and Fukunaga \cite{asai_tie-breaking_2017} it has become standard practice to break ties in favor of nodes with lower $h$-value for a better performance. They did experiments on multi level tie-breaking, which also resolves ties between nodes that have the same $f$- and $h$-value. These experiments were not specific to Flood-It and showed that, for their set of problems, a lifo\footnote{%
LIFO = last in first out (stack); FIFO = first in first out (queue)
} 
policy $[f, h, \text{lifo}]$ outperforms a $[f, h, \text{fifo}]$ policy and interestingly $[f, \text{lifo}]$ performed nearly as well as $[f, h, \text{fifo}]$.

\section{Finding Solutions with General Solvers}
\label{sec:solvers}
Formulating puzzle rules as declarative constraints and utilizing general-purpose solvers, such as Boolean Satisfiability (SAT), Integer Linear Programming (ILP), or Constraint Programming (CP), is an established method for combinatorial problems. However, while state-based puzzles can often be modeled by defining a single valid goal configuration, process-based puzzles like Flood-It present a distinct challenge. Because the solution is a sequence of state transitions, the temporal dimension must be explicitly encoded.

To model the progression of Flood-It, we employ a time-unrolling approach. The variables representing the game state and the executed moves are duplicated across a discrete time horizon, parameterized by an estimated or bounded number of moves, $T$. Consequently, the size of the variable space and the number of required constraints scale with $T$.
Modeling processes over time inherently introduces redundancies within the search space, as solvers may explore functionally equivalent sequences or delay mandatory state changes. This can be addressed by adding symmetry-breaking constraints to the base rules of the game. These optional constraints are designed to restrict the solver to canonical evaluation paths, for instance, formally ensuring that a vertex remains flooded in all subsequent time steps once captured, or requiring that all eligible neighboring components are flooded simultaneously. Symmetry-breaking constraints may, or may not improve the efficiency of the solving process.

In the following subsections, we define a unified set of constraints for the Flood-It problem and detail their specific adaptations into the respective domains of CP, ILP, and SAT.

\subsection{A Unified Model for Flood-It}
\label{sec:solvers|model}
To formulate Flood-It for declarative solvers, we define a core set of constraints modeling graph dynamics over a discrete time horizon $T$.
Let $V$ be the number of vertices, $C$ the number of colors, and $T$ the bounded number of moves. We define $adj[v]$ as the set of adjacent vertices to $v$, and $color[v]$ as the color of vertex $v$. Vertex $0$ is designated as the pivot (from which flooding originates).

The game dynamics are modeled using the following variables:
\begin{itemize}
    \item $\Fld(t,v)$: Indicates whether vertex $v<V$ is flooded at time step $t\le T$.
    \item $\Move(t,c):$ Indicates whether color $c<C$ is selected at time step $t<T$ (one-hot encoding).
\end{itemize}

Any valid solution must satisfy the base constraints for initial and final states \eqref{eq:pivot} to \eqref{eq:final}, and the constraints for move selection and dynamics \eqref{eq:one_move} to \eqref{eq:flood_adj}:

\begin{align}
    & \quad \text{fld}(0, 0) = 1 \label{eq:pivot} \\
    \forall_{0<v<V} & \quad \text{fld}(0, v) = 0 \label{eq:not_flooded_0} \\
    \forall_{v<V} & \quad \text{fld}(T, v) = 1 \label{eq:final} \\
    \forall_{t<T} & \quad \sum_{c<C} \text{move}(t, c) = 1 \label{eq:one_move} \\
    \forall_{t<T}\forall_{v<V} & \quad \text{fld}(t+1, v) \le \text{fld}(t, v) + \text{move}(t, color[v]) \label{eq:flood_color} \\
    \forall_{t<T}\forall_{v<V} & \quad \text{fld}(t+1, v) \le \text{fld}(t, v) + \sum_{u \in adj[v]} \text{fld}(t, u) \label{eq:flood_adj}
\end{align}

Any assignment of the variables $\Move$ satisfying those constraints is a valid solution of the given Flood-It instance. Note that, however, the variable $\Fld$ does not necessarily correspond to the respective flooding state, as the flooding operation is not fully described by the constraints. It might, for example, happen that a vertex that was flooded already becomes unflooded at a later time step, as long as it will get flooded again eventually. Similarly, a vertex that should get flooded according to the flooding rules might not be marked as flooded as long as it will get flooded later. Our constraints so far ensure only that \emph{nothing more} is flooded than it should be, according to the rules. Considering efficiency, we consider the following optional symmetry-breaking constraints:

\begin{align}
    \forall_{t<T}\forall_{v<V}  & \quad \Fld(t,v) \le \Fld(t+1,v)                         \label{CP|floodedforever}\tag{P}\\
    \forall_{t<T}\forall_{v<V}\forall{u\in adj[v]}  & \quad \Fld(t,v) + \Move(t, color[u]) \leq \Fld(t+1,u) + 1       \label{CP|exhaustive}\tag{E}
\end{align}

\subsection{Paradigm-Specific Adaptations}
\label{sec:solvers|specific}
\begin{itemize}
  \item Integer Linear Programming (ILP): Uses linear inequalities over binary variables to represent logical implications as shown in the formulation above.
  \item Constraint Programming (CP): Uses native integer domains for $\Move(t)$ ranging from $0$ to $C-1$, replacing the one-hot binary array. Boolean operations and logical implications can be expressed directly without the transformation into linear inequalities. For example, constraints (E) are modeled as $(\Fld(t,v) \wedge \Move(t) = color[u]) \Rightarrow \Fld(t+1,u)$.
  \item Boolean Satisfiability (SAT): Converts all inequalities and implications into Conjunctive Normal Form (CNF) clauses using standard clausal expansions.
\end{itemize}

\subsubsection{Minimization Problem with CP and ILP}
\label{sec:solvers|cp|minimization}
Using CP and ILP also provides the opportunity to add a minimization to determine a solution with the minimum number of moves. However, since the size of the variable arrays $\Fld$ and $\Move$ depends on the number of moves to be modeled, we still need an estimate of an upper bound for the number of moves, which we denote by $T$. Below, we will add a constraint that forces only the first moves to be used whenever a solution exists that has less than $T$ moves. For the actual minimal number of moves, we introduce an integer variable $\Tmin$. Furthermore, we enlarge the domain of the variables $\Move$ to include the value $-1$, which indicates that no flooding is performed. 

The new CP constraint, which makes sure that only the first $\Tmin$ moves are used, while later moves are void (non-flooding) moves is:
\begin{align}
    \forall_{t<T} & \quad \Move(t)\ge 0 \Rightarrow \Tmin\ge t                          \label{CP|min}\tag{M}
\end{align}
When the CP solver is set to minimize $\Tmin$, a shortest solution will be returned. 

For ILP, the corresponding constraint is
\begin{align}
    \forall_{t<T}\forall_{c<C} & \quad \Move(t,c)*t\le \Tmin                          \label{ILP|min}\tag{M}
\end{align}
with the integer variable $\Tmin$ that is set to be minimized.

In case the bound $T$ was selected too small, the solver will fail to find a solution. In that case, the process can be restarted with an enlarged bound. Having a good estimate for the number of moves will be even more important when using the SAT approach (\autoref{sec:SAT|minimization}).

\subsubsection{Minimization with the SAT Approach}
\label{sec:SAT|minimization}
With a SAT solver, we can only get a solution for the decision problem, i.e., given a number of moves $T$. In order to determine a minimum length solution, the decision solver needs to be called repeatedly until we find a value $T$ such that a solution with $T$ moves exists, but not so for $T-1$ moves. To keep the number of iterations for this search low, a tight estimate of the minimum solution length is helpful. We will determine such an estimate as a regression of the solution length from the board size and the number of colors from Flood-It training instances, see \autoref{sec:solvers|prediction}.

There may be more efficient strategies than using a fixed step size of $1$ for the iterative search. However, our results suggest that for random puzzle instances with the board sizes and number of colors investigated here, there might not be much to gain. Anticipating later results, \autoref{fig:sol-len-stats} shows there is not too much spread in the minimum solution lengths for a fixed combination of \emph{size} and \emph{colors}. Moreover, \autoref{app:solving_time_vs_bound} in the appendix indicates that solving times are typically relatively low in the decision solver, when the given solution length is off from the minimum by more than $2$ or $3$, be it too low or too high.

\subsection{Estimating the Time Horizon}
\label{sec:solvers|prediction}
To use a SAT solver to determine a minimum solution, repeated calls to the decision solver are required while probing different solution lengths for feasibility. When a number $k$ is found such that a Flood-It solution with $k$ moves exists, but not so for $k-1$, the minimum number of moves is determined as $k$. For this search process, an accurate prediction of the minimum solution length is beneficial to minimize solver iterations. 

Similarly, the CP and ILP approaches using minimization require an estimate of the solution length. The formulations involve a coding of the flooding operations where the number of required variables scales with the bounded time horizon $T$. 

Stiphout \cite{vanstiphoutFlooditSatProblem2020} observed that the shortest solution length for a Flood-It board with $n \times n$ tiles and $k$ colors falls between $n + k - 5$ and $n + k$. This observation was based on randomly generated boards with $n \in \{2, \dots, 9\}$ and $k \in \{3, \dots, 6\}$. To extend this to larger values of $n$ and $k$, we analyzed the distribution of minimum solution lengths for randomly generated puzzle instances.

For puzzles of a fixed size and a fixed number of colors, the distribution of minimum solution lengths typically falls into a window of width 4 to 6, see \autoref{fig:sol-len-stats}. The dependency of the solution length on \emph{size} and \emph{colors} is approximated with a linear function, see \autoref{fig:sol-len-reg-size} and \ref{fig:sol-len-reg-colors} in the appendix.

We propose a linear regression from the parameters \emph{size} (side length of a quadratic board) and \emph{colors} to predict the minimum number of moves required. The regression parameters were estimated using 100 randomly generated puzzle instances per \emph{size} (12, 14, 16, 18, 20) and \emph{colors} (5, 6, 7, 8, 9, 10). The ordinary least squares (OLS) regression on this training data is given by (rounded to one decimal place):
\begin{equation}
    \text{predicted solution length} = 2.4 \cdot\textit{colors} + 1.4\cdot \textit{size} - 14.3
    \label{eqn:regression}
\end{equation}
see also \autoref{fig:sol-len-reg} and \autoref{tab:sol-len-regression}. We take the rounded value of \autoref{eqn:regression} as the initial guess for $T$ in the SAT iterative search, and as the bounded time horizon for the CP and ILP variable arrays.

\begin{figure}
    \centering
    \includegraphics[width=0.6\linewidth]{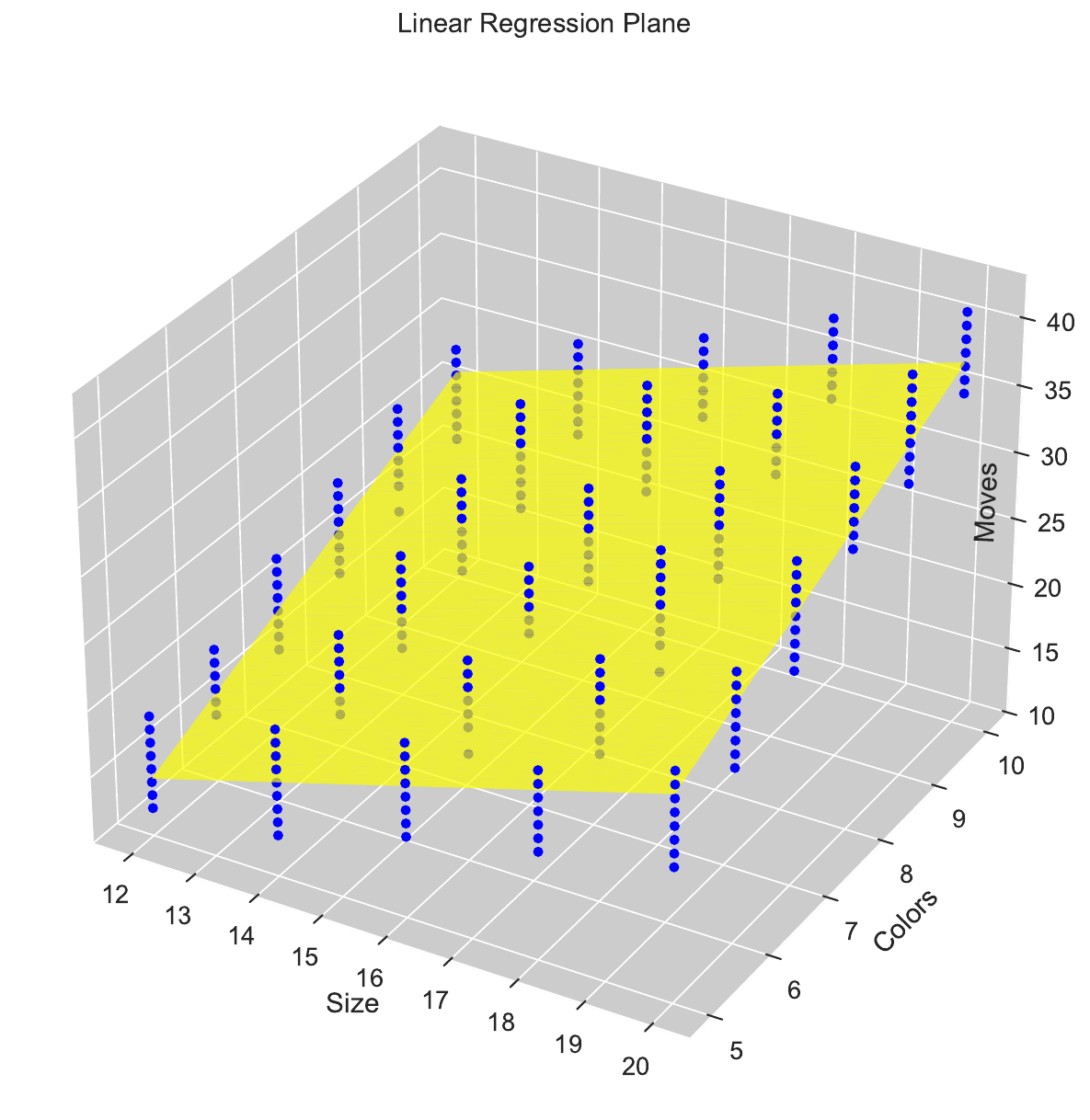}
    \caption{\textbf{Regression of solution length from size and colors.} The points indicate the different solutions lengths that occurred in randomly generated puzzle instances for the given size and number of colors. The plane is a regression of the solution length (\emph{moves}) from \emph{size} and \emph{colors}. Note, that this visualization does not indicate the distributions of \emph{moves}. Those can be seen in the projections, Appendix~\ref{sec:app}, \autoref{fig:sol-len-reg-colors} and \ref{fig:sol-len-reg-size}.}
    \label{fig:sol-len-reg}
\end{figure}

\begin{table}[!hb]
    \begin{subtable}{0.45\textwidth}
        \centering
\begin{tabular}{c@{}D{.}{.}{0}rrrrr}
\toprule
size & & 12 & 14 & 16 & 18 & 20 \\
\midrule
\multirow{6}{*}{\rotatebox{90}{\#colors}} 
 & 5  & 14.3 & 17.0 & 19.8 & 22.6 & 25.3 \\
 & 6  & 16.7 & 19.4 & 22.2 & 25.0 & 27.7 \\
 & 7  & 19.1 & 21.8 & 24.6 & 27.3 & 30.1 \\
 & 8  & 21.5 & 24.2 & 27.0 & 29.7 & 32.5 \\
 & 9  & 23.9 & 26.6 & 29.4 & 32.1 & 34.9 \\
 & 10 & 26.3 & 29.0 & 31.8 & 34.5 & 37.3 \\
\bottomrule
\end{tabular}
\caption{Values predicted by regression \eqref{eqn:regression}.}
        \label{tab:sol-len|predicted}
    \end{subtable}
    \hfill
    \begin{subtable}{0.45\textwidth}
        \centering
\begin{tabular}{c@{}D{.}{.}{0}rrrrr}
\toprule
size & & 12 & 14 & 16 & 18 & 20 \\
\midrule
\multirow{6}{*}{\rotatebox{90}{\#colors}} 
& 5 & 15.5 & 17.0 & 19.5 & 21.0 & 23.5 \\
& 6 & 17.5 & 21.0 & 22.0 & 24.5 & 26.5 \\
& 7 & 19.5 & 22.5 & 25.5 & 28.0 & 30.0 \\
& 8 & 21.5 & 24.5 & 26.5 & 30.0 & 34.0 \\
& 9 & 23.5 & 26.0 & 30.0 & 33.0 & 36.0 \\
& 10 & 24.5 & 27.5 & 31.0 & 35.0 & 38.0 \\
\bottomrule
\end{tabular}
        \caption{Actual minimum solution lengths.}
        \label{tab:sol-len|actual}
    \end{subtable}
    \caption{\textbf{Prediction of the minimum solution length.}
    \textbf{(a)}~The minimum solution length for Flood-It puzzles, predicted by OLS regression from the \emph{size} (side length of a quadratic grid) and number of \emph{colors}. In \textbf{(b)} the median across 100 puzzle instances was calculated per size and color. Note, that this is the same data as the OLS was trained on.}
    \label{tab:sol-len-regression}
\end{table}

\section{Experiments}
\label{sec:experiments}

\subsection{Generation of Flood-It instances}
\label{sec:experiments|instances}
We generated Flood-It instances on square grids of different sizes (side lengths 12, 14, 16, 18, 20) and different numbers of colors (5 to 10). For each cell of the grid, the color was independently determined by Python's (pseudo) random generator \texttt{random.randint()}. We verified that all possible colors were actually used. One hundred instances were generated for each combination of size and number of colors. \autoref{fig:sol-len-stats} shows the distribution of the minimum solution lengths among the generated puzzles. 
\begin{figure}
    \centering
    \includegraphics[width=\linewidth]{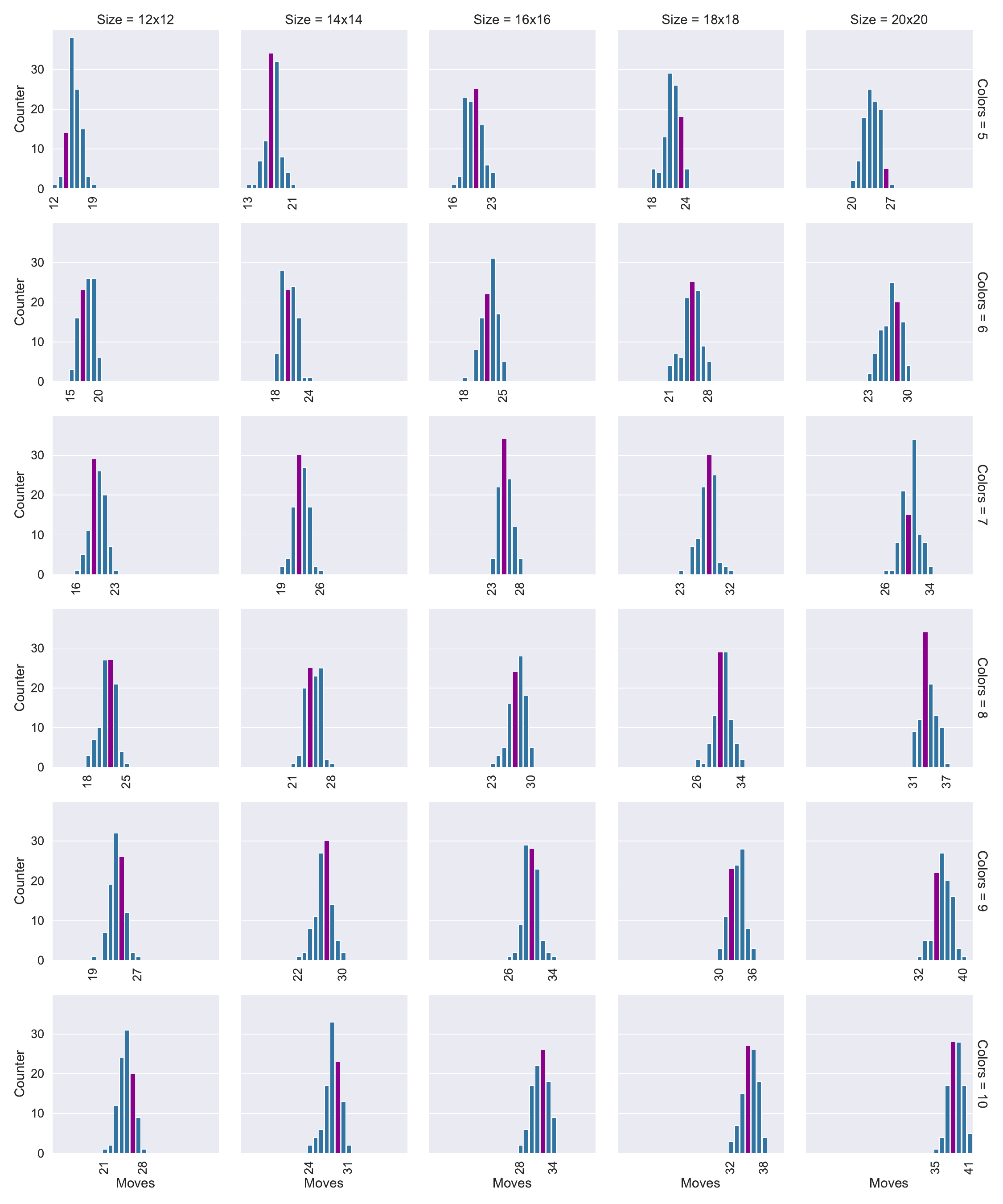}
    \caption{For each \emph{size} and number of \emph{colors}, 100 Flood-It instances were randomly created. The bar plots show the distribution of the minimum solution lengths for each pair of \emph{size} and \emph{colors}. The bar corresponding to the \emph{predicted} minimum solution length based on a regression from \emph{size} and \emph{color}, see \autoref{sec:solvers|prediction}, is shown in dark magenta.}
    \label{fig:sol-len-stats}
\end{figure}

\paragraph{Training and test set}
To obtain a representative dataset for the final evaluation of the solving algorithms (`test set'), we picked the first instance of each \emph{size}/\emph{color} combination that had a minimal solution length not encountered before. This process resulted in a test set containing 233 Flood-It instances, one from each bin in \autoref{fig:sol-len-stats}. 
For tuning each approach, we used a different dataset of 99 puzzle instances (`training set'), which was created earlier in a similar way.

\subsection{Implementation of Solver Approaches}
\label{sec:experiments|implementation}
The A* approach was implemented in C++, whereas we used Python interfaces to different solvers to implement the approaches introduced in \autoref{sec:solvers}. For SAT solvers, we used the freely available PySAT toolkit \cite{ignatiev2018pysat} which provides an interface to a number of dedicated SAT solvers. As an alternative, we implemented the SAT constraints in the much more general Z3 Satisfiability Modulo Theories (SMT) solver developed by Microsoft \cite{de2008z3} as well as using the open-source optimization suite OR-Tools\cite{ortools} developed by Google. For ILP, we again used Z3, OR-Tools, and additionally Gurobi \cite{gurobi} and CPLEX \cite{cplex,bixby2002cplex}, which are both free for academic use. For Constraint Programming, we used CPLEX, OR-Tools and the Z3 solver. Experiments were run on a MacBook Pro with an M2 Max chip and 96~GB memory.

\section{Results}
\label{sec:results}
We first discuss individual results for the decision problem. As a high number of variations have been investigated in this study, we show just the most important results in the main text, while more detailed comparisons and results of less successful variants are either included in the appendix or are just shortly mentioned.

Since it turned out that the SAT Solvers showed a much better performance compared to ILP and CP, we present the results in more detail for them.

We evaluated the different SAT solvers provided in the PySAT toolkit for the Flood-It decision problem, see \autoref{app:results|sat-solvers} in the appendix. Based on the results, we selected Glucose 4.2.1 (\solver{g42}) as the best performing one for subsequent simulations.

\subsection{Comparison of Enhancements for A*}
\label{sec:results|astar}
In this section, we evaluate the possible optimizations for A* that were introduced in \autoref{sec:astar} using the 99 instances of the training set. \autoref{fig:survival-Astar-overview} shows an overview of adding one-by-one the best option for each enhancement (heuristics, game state equivalence, pruning). The solving times for each solver configuration are displayed as survival plot \cite{brain2017benchmarking}. There is a marked improvement with each enhancement, except for the pruning strategy.
\begin{figure}
    \centering
    \includegraphics[width=0.5\linewidth]{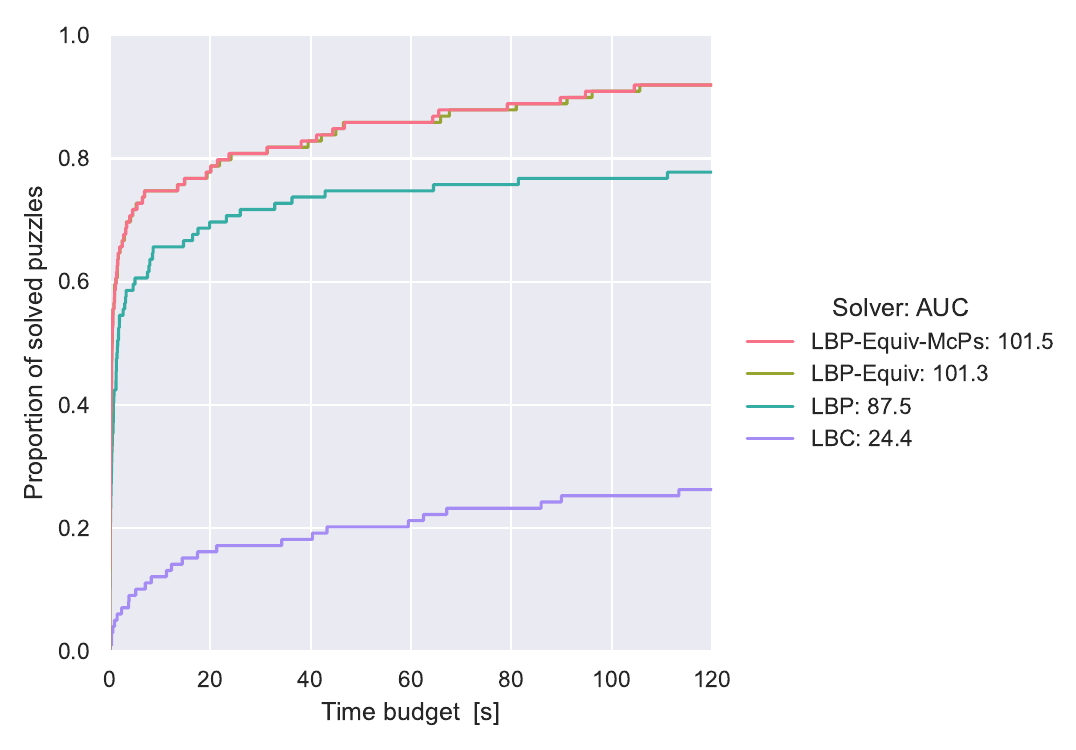}
    \caption{\textbf{A* Enhancements.} Comparison of the solution times for different A* configurations in the \emph{minimization} problem as survival plot. For each time budget on the $x$-axis up to the timeout limit of 120~s, the curves represent the ratio of puzzles solved within that time limit. The values in the legend specify the area under the curve (AUC) as an overall measure of performance. $LBC$ uses simply the number of unsolved colors as lower bound, and $LBP$ uses \autoref{algo:lower_bounds_puchert}. LBP-Equiv exploits additionally the equivalence of game states. LBP-Equiv-McPs furthermore uses the ``must check -- presolved'' strategy for pruning. For tie-breaking [$f$, $h$, lifo] was used in all variants. 
    Data: 99 instances of the training set, see \autoref{sec:experiments|instances}.}
    \label{fig:survival-Astar-overview}
\end{figure}
In this overview and the individual look at each enhancement in the subsequent paragraphs, the strategy [$f$, $h$, lifo] was used for tie-breaking, as it is recommended generally. In the last paragraph of this section, different strategies for breaking ties are compared.

\paragraph{Lower Bounds as Heuristics}
We validated the two lower bounds introduced in \autoref{sec:astar|lower_bounds} and included as a baseline in the comparison the heuristic which uses the number of unsolved colors as simple lower bound (LBC), see \autoref{fig:astar|lb|survival}. Both more sophisticated lower bounds ($LBL$, $LBP$) increase the efficiency of A* substantially compared to using a simple lower bound. The competition between the better two is shown in \autoref{fig:astar|lb|scatter} as a scatter plot. The lower bound $LBP$ (\autoref{algo:lower_bounds_puchert}), which takes safe moves into account additionally to the level structure, results in a speedup of up to factor 100 and is faster compared to the lower bound $LBL$ (\autoref{def:lbls}) in all the cases that were solved by either method.
\begin{figure}
    \begin{subfigure}{0.48\linewidth}
        \includegraphics[width=\linewidth]{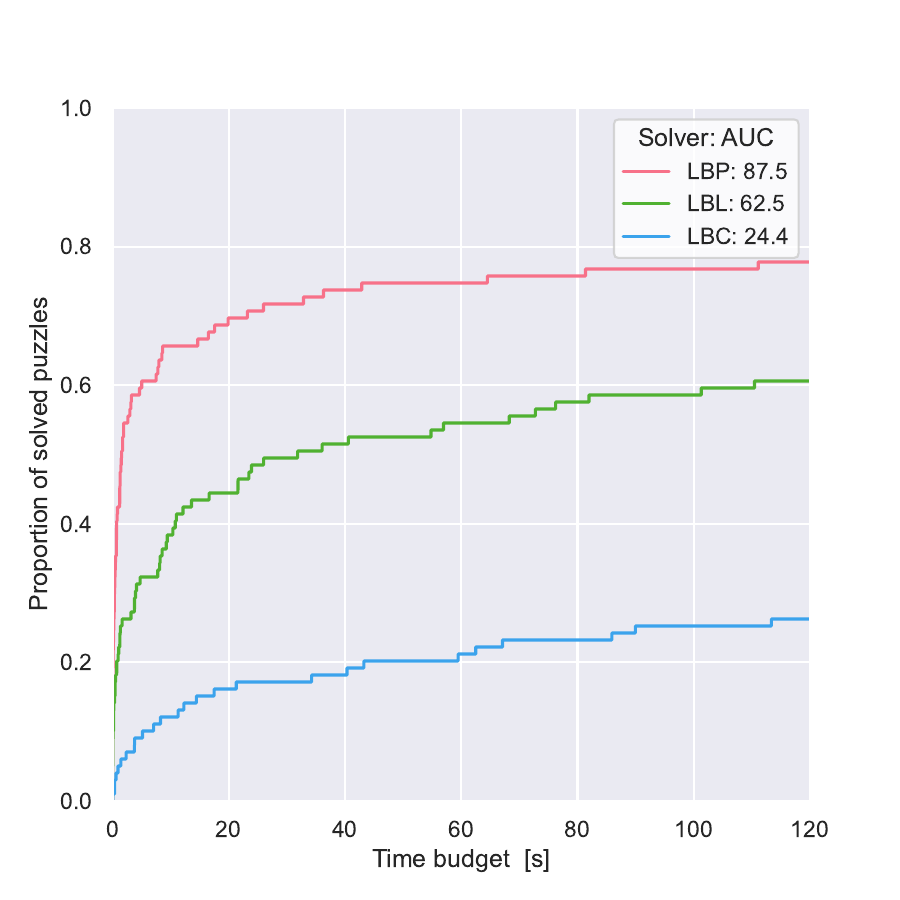}
        \caption{Survival plot for $LBC$, $LBL$, $LBP$}
        \label{fig:astar|lb|survival}
    \end{subfigure}
    \hfill
    \begin{subfigure}{0.50\linewidth}
        \includegraphics[width=\linewidth]{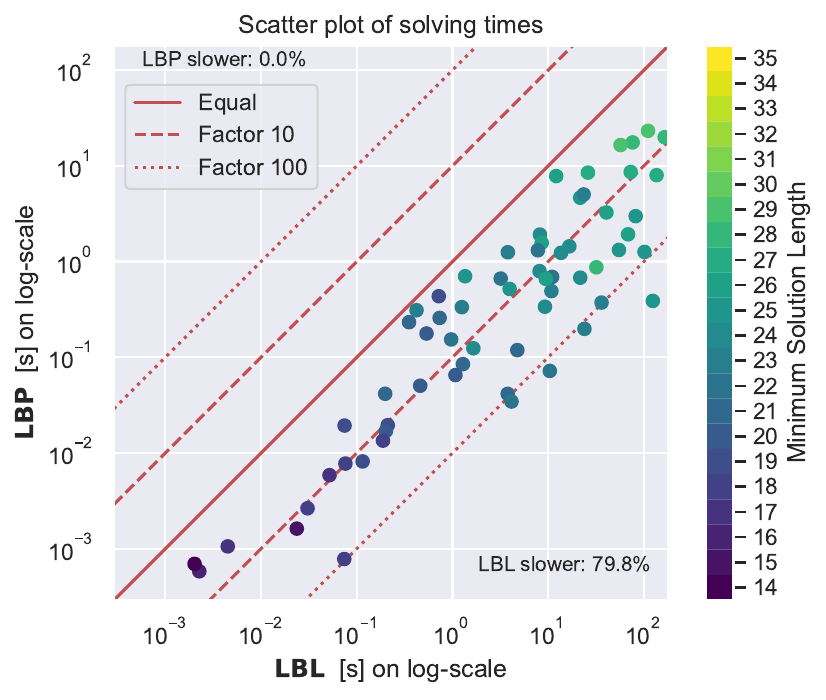}
        \caption{Scatter plot for $LPL$ vs.\@ $LBP$}
        \label{fig:astar|lb|scatter}
    \end{subfigure}
    \caption{\textbf{Comparison of different lower bounds for A*.} \textbf{(a)}~For each time budget on the $x$-axis up to the timeout limit of 120~s, the curves represent the ratio of puzzles solved within that time limit. The more sophisticated lower bounds $LBL$ and $LBP$ introduced in \autoref{sec:astar|lower_bounds} are compared to the simple lower bound of the number of unsolved colors ($LBC$). \textbf{(b)}~The better two heuristics are compared in a scatter plot. For points below the diagonal, $LBL$ is slower. The minimum solution length of each puzzle instance is color-coded. Note that 20.2\% of the instances were not solved by either method.
    Data: 99 instances of the training set, see \autoref{sec:experiments|instances}.}
    \label{fig:astar|lb}
\end{figure}
For further evaluations, we employed the lower bound $LBP$ as heuristic.

\paragraph{Using Game State Equivalences}
Making use of the equivalence introduced in \autoref{theorem:game_states_equal_min} is discussed in \autoref{sec:astar|equivalent_game_states} for pruning the search space. This results in more solved test instances and a faster solving time in $94.9\%$ of the cases with a speedup factor of up to 10, see \autoref{fig:scatterplot-Astar-equiv-equal}. The advantage gets more prominent with the increasing `complexity' of the puzzle.
\begin{figure}
    \begin{subfigure}{0.515\linewidth}
        \includegraphics[width=\linewidth]{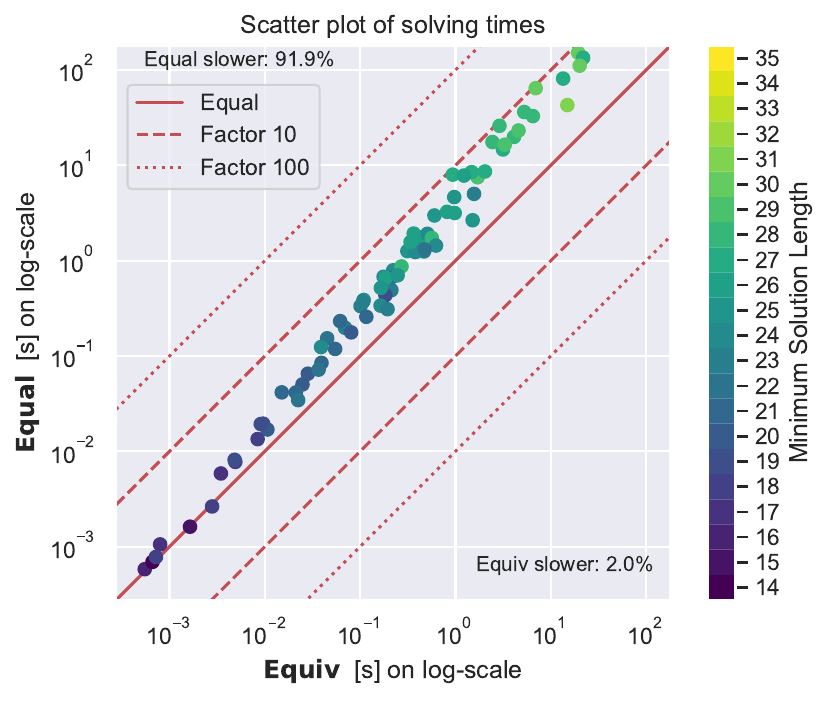}
        \caption{Scatter plot for equivalence}
        \label{fig:scatterplot-Astar-equiv-equal}
    \end{subfigure}
    \hfill
    \begin{subfigure}{0.465\linewidth}
        \includegraphics[width=\linewidth]{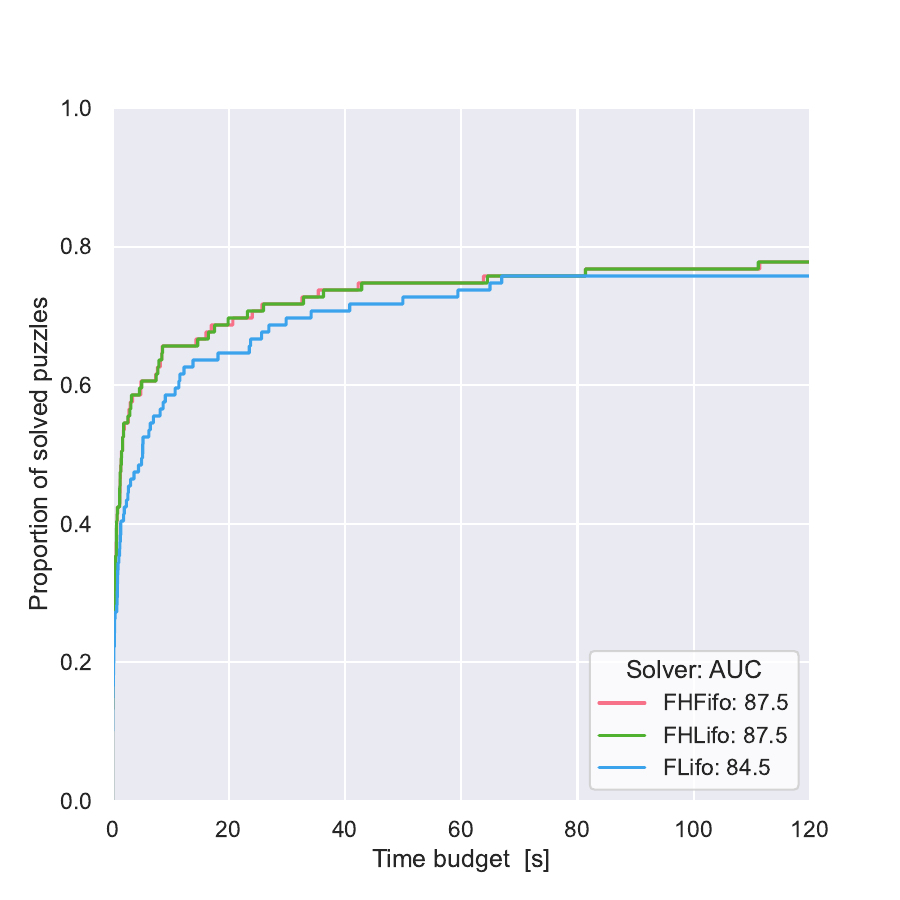}
        \caption{Survival plot for tie-breaking}
        \label{fig:astar|ties|survival}
    \end{subfigure}
    \caption{\textbf{(a)~Effect of exploiting equivalence of game states.}
    In `Equal' only exactly equal game states are identified, while in `Equiv' the equivalence of game states is used for pruning the search of A*, see \autoref{sec:astar|equivalent_game_states}.
    For points above the diagonal, `Equal' is slower. The minimum solution length of each puzzle instance is color coded. 
    \textbf{(b)~Comparison of tie-breaking strategies.} Ties of game states with the same $f$-value are resolved either by a `lifo' selection (FLifo), or choosing the lowest $h$ value and on the third level either a `lifo' (FHLifo) or a `fifo' (FHFifo) selection, see \autoref{sec:astar|tie_breaking}.
    Data: 99 instances of the training set, see \autoref{sec:experiments|instances}.}
\end{figure}
For further evaluations, we made use of the equivalence.

\paragraph{Pruning Strategies}
In our simulations, the benefit of the pruning strategies of \autoref{sec:astar|branching_stragegies} vanished when the other optimizations (equivalent game states and $LBP$) are employed. We used McPs in the final evaluation, as it gave the best results on the training data by a small margin (\autoref{fig:survival-Astar-overview}).

\paragraph{Tie Breaking}
\label{sec:results|astar|tiebreaking}
Finally, we take a look at the different options for tie-breaking. Selecting game states with the lowest $h$-value among those having the same $f$-value was found to be beneficial over a `lifo' selection.
Comparing the strategies that both use the $h$-value to break ties on the second level, [$f$, $h$, lifo] performs slightly better than [$f$, $h$, fifo], see \autoref{fig:astar|ties|survival}.

The direct comparison between [$f$, $h$, lifo] and [$f$, lifo]
shows a clear benefit of breaking ties based on $h$-value on the second level with [$f$, $h$, lifo] being faster on $86.9\%$ of the test instances with a speedup factor of up to $10$.

\subsection{Comparison of Different SAT Formulations}
\label{sec:results|sat-formulations}
The optional constraint (E) increased solving times in most cases. Also, it led to a substantial increase in variance in solving times, see \autoref{app:results|variation} in the appendix. Therefore, this option is not further discussed here. Moreover, an alternative formulation proposed in \cite{vanstiphoutFlooditSatProblem2020} (\autoref{app:solvers|satv2} in the appendix), which introduces an additional variable for the concept of vertices ‘touching the flood’ was found to be less efficient in our setting (\autoref{app:results|SATv2}).

Within the SAT solver available in the PySAT toolkit we selected Glucose 4.2.1 (\solver{g42}) based on the results shown in \autoref{app:results|sat-solvers}. We compared that with the SAT formulations within the more general Z3-Solver and OR-Tools. The solving times for the different constraint variants are displayed in \autoref{fig:comp-SAT-variants} similar to a survival plot \cite{brain2017benchmarking} but without accumulation of times on the $x$-axis.  In all three solvers, using the optional constraint (P) with or without constraint (C) is beneficial. 
\begin{figure}[ht!]
    \centering
    \begin{subfigure}{0.32\textwidth}
        \includegraphics[width=\linewidth]{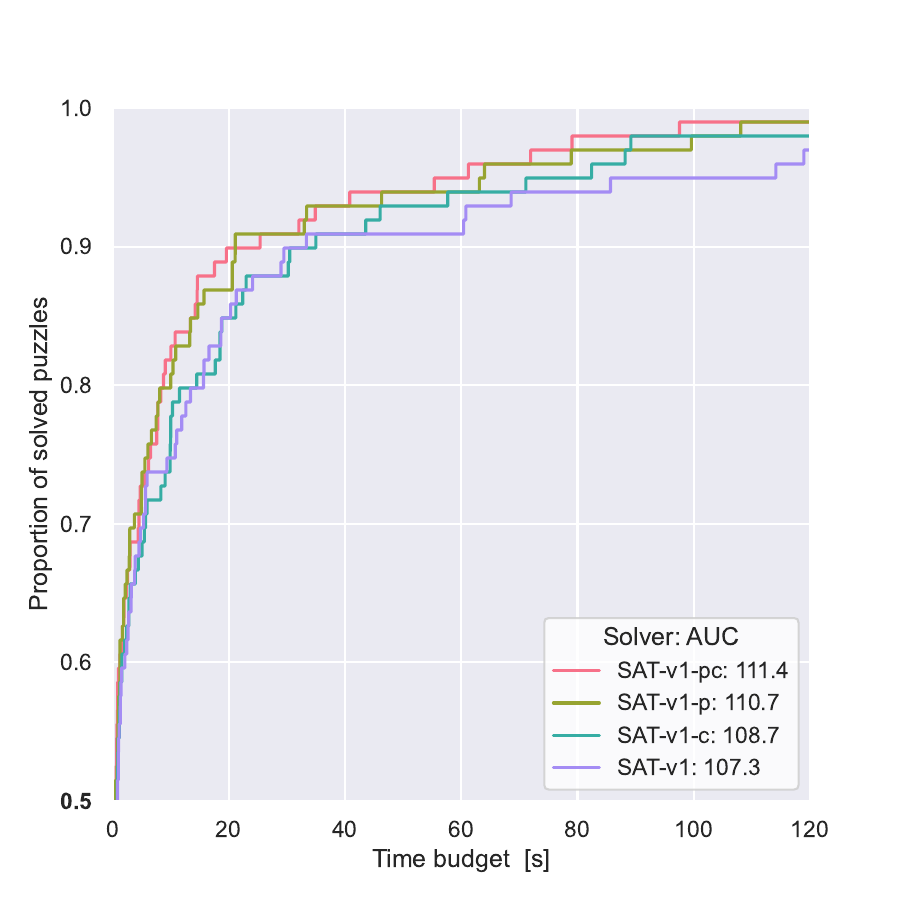}
        \caption{Glucose 4.2.1 SAT solver.}
        \label{fig:comp-SAT-variant|g42}
    \end{subfigure}
    \hfill
    \begin{subfigure}{0.32\textwidth}
        \includegraphics[width=\linewidth]{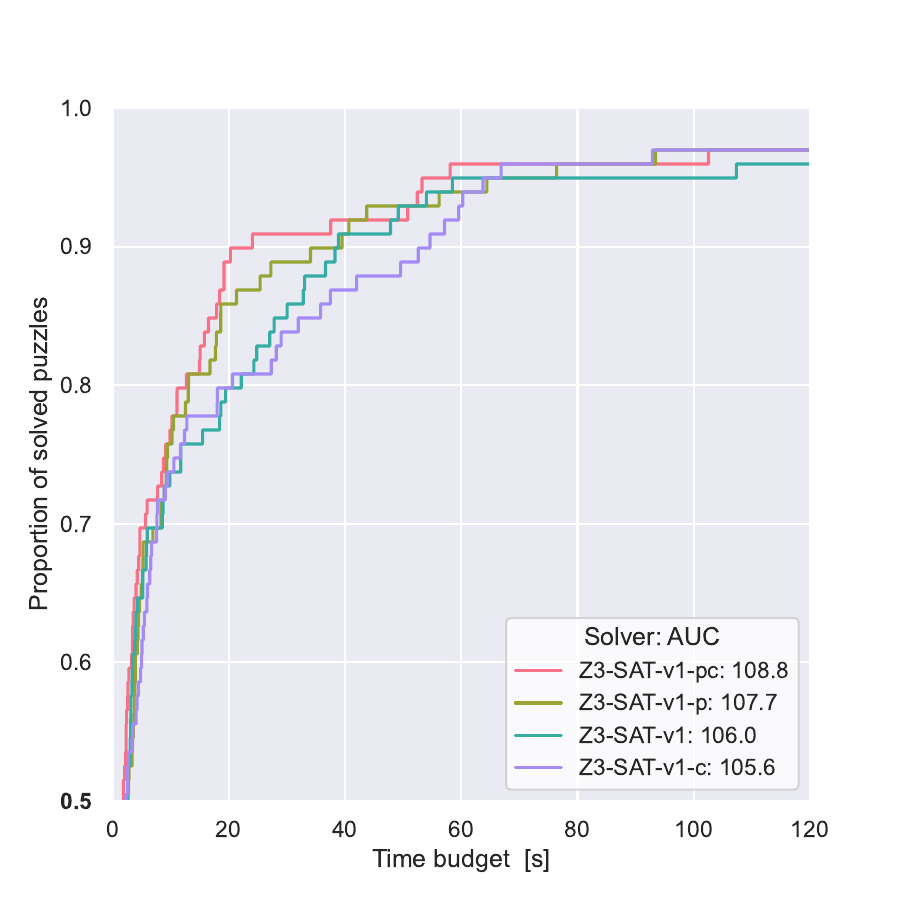}
        \caption{SAT formulation in Z3-Solver.}
        \label{fig:comp-SAT-variant|Z3}
    \end{subfigure}
    \hfill
    \begin{subfigure}{0.32\textwidth}
        \includegraphics[width=\linewidth]{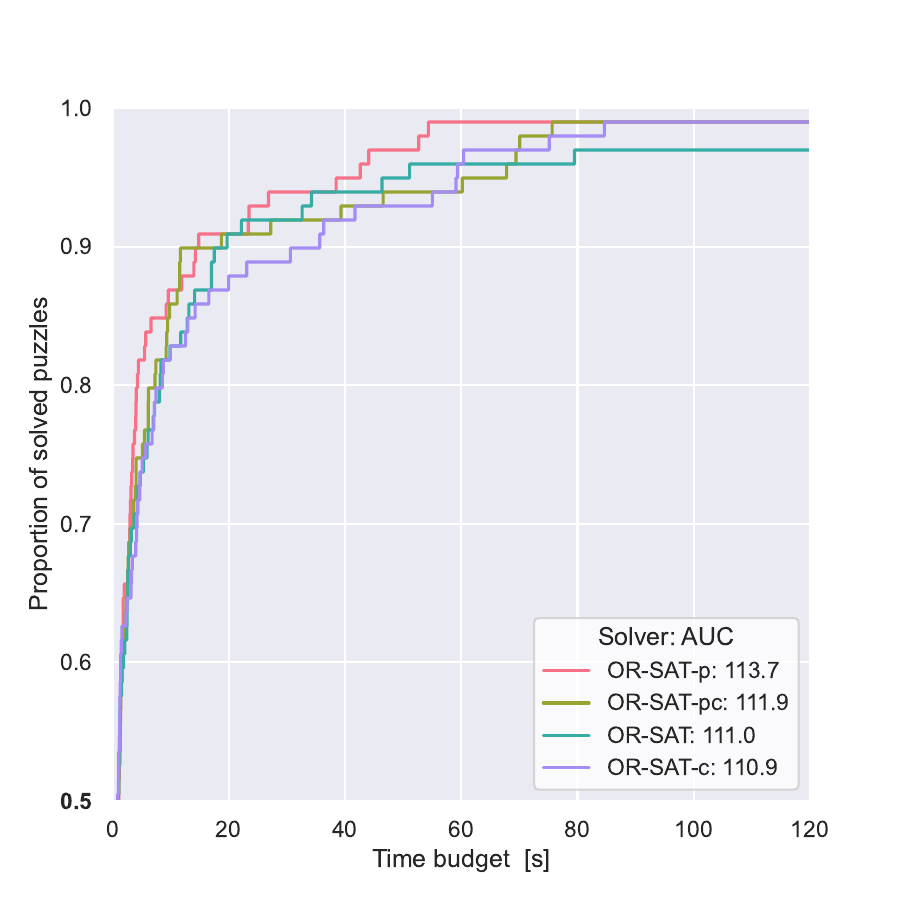}
        \caption{SAT formulation in OR-Tools.}
        \label{fig:comp-SAT-variant|OR}
    \end{subfigure}

    \caption{\textbf{Different solver engines for SAT.} Comparison of the solution times when using the SAT formulation with different solver engines and constraint variants in the \emph{decision} problem as survival plot. For each time budget on the $x$-axis up to the timeout limit of 120~s, the curves represent the ratio of puzzles solved within that time limit. The values in the legend specify the area under the curve (AUC) as an overall measure of performance. 
    Data: 99 instances of the training set, see \autoref{sec:experiments|instances}.}
    \label{fig:comp-SAT-variants}
\end{figure}
The comparison between the solvers bears OR-Tools as the winner, closely followed by the Glucose solver of PySAT and finally Z3. This order is replicated on the test set, see next section.

\subsection{Comparison of Solvers for the Decision Problem}
\label{sec:results|decision-solvers}
Based on the investigation of the 99 instances of the training set presented in \autoref{sec:results|sat-formulations}, we selected the best variant for each solver using SAT formulations. Furthermore, we included the best A* configuration in the final comparison for the decision problem. The results using the instances of the test set are shown as survival plot in \autoref{fig:comp-decision}.
\begin{figure}
    \centering
    \begin{subfigure}{0.49\textwidth}
        \includegraphics[width=\linewidth]{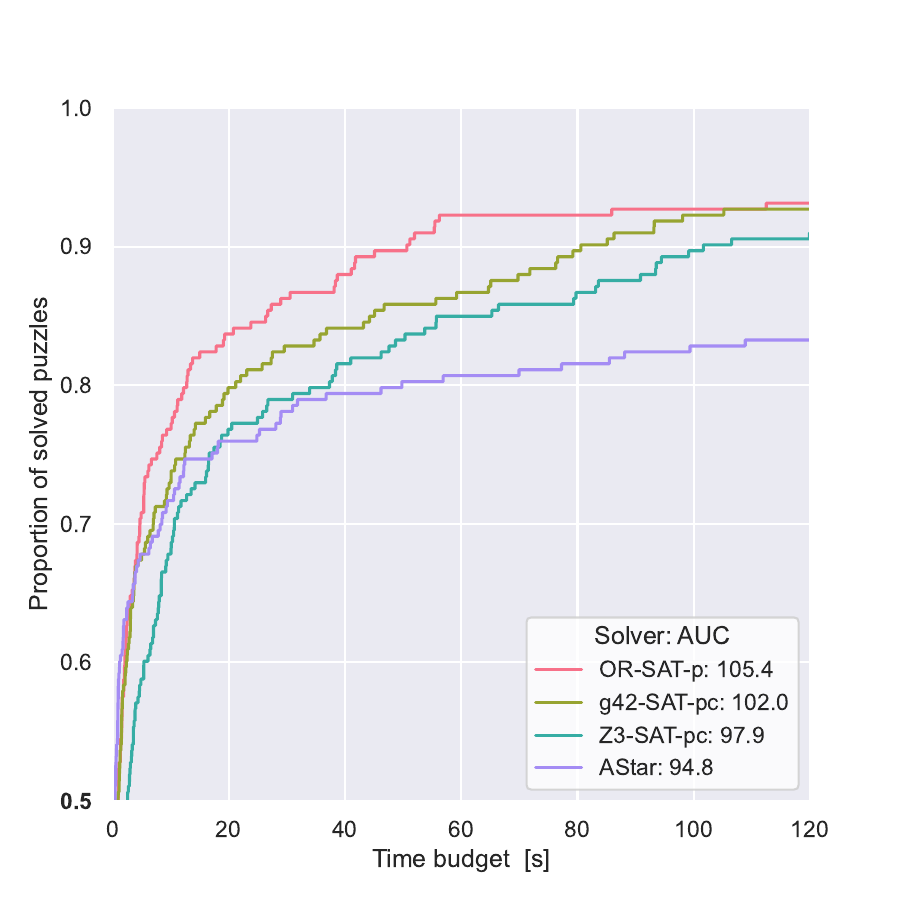}
        \caption{Survival plot with linear scaling.}
        \label{fig:comp-dec|testset}
    \end{subfigure}
    \hfill
    \begin{subfigure}{0.49\textwidth}
        \includegraphics[width=\linewidth]{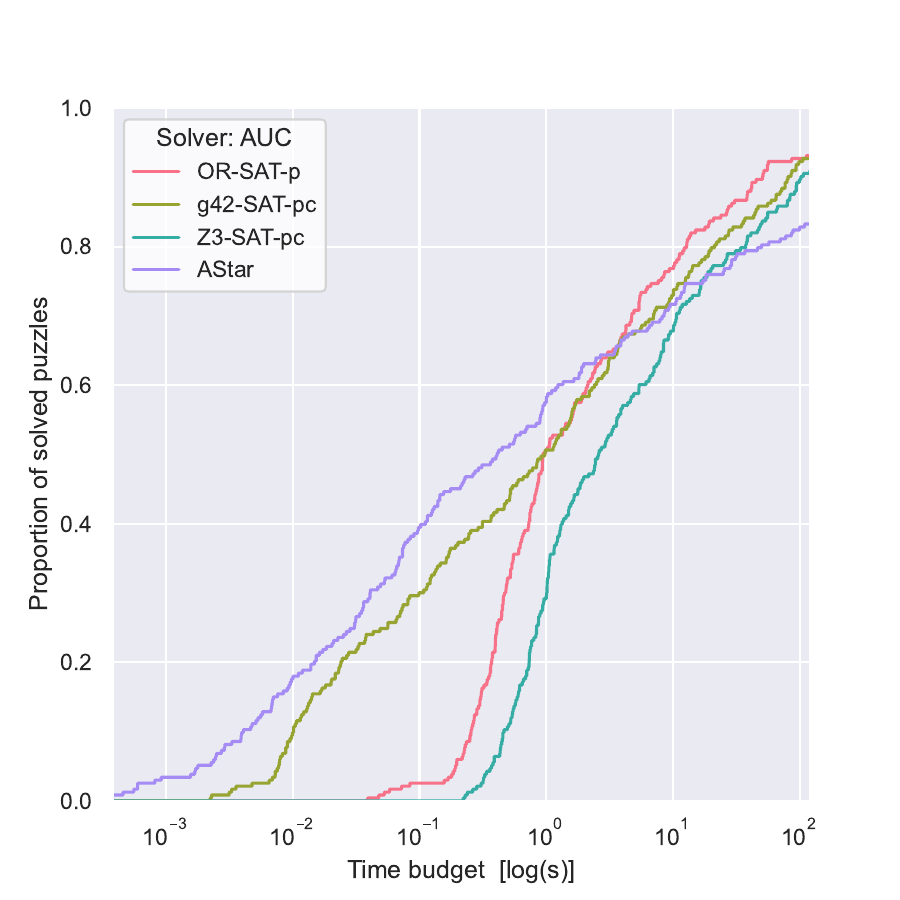}
        \caption{Survival plot with logarithmic scaling.}
        \label{fig:comp-dec|testset-log}
    \end{subfigure}
    
    \caption{\textbf{Flood-It decision problem.} Comparison of the solution times for the solvers in the \emph{decision} problem as survival plot. For each time budget on the $x$-axis up to the timeout limit of 120~s, the curves represent the ratio of puzzles solved within that time limit. The scaling on the $x$-axis is linear in \textbf{(a)} and logarithmic in \textbf{(b)}. Note, that the $y$-axis starts at $0.5$ in the left subplot. Data: 233 instances of the test set, see \autoref{sec:experiments|instances}.}
    \label{fig:comp-decision}
\end{figure}

In the view with linear scaling in \autoref{fig:comp-dec|testset}, the SAT solvers dominate A* search. Within the SAT solvers, the same order as on the training set is found, compare \autoref{fig:comp-SAT-variants}. Due to the wide range of the time budget from 0 to 120 seconds, and the solving times for the majority of instances falling into the subsecond range, the plot with a linear scaling on the $x$-axis gives an impression that is mainly driven by the less frequent hard instances. The plot with logarithmic scaling on the $x$-axis in \autoref{fig:comp-dec|testset-log} reveals that A* search performs best for the majority of puzzle instances that are solved within the first few seconds. However, as the allotted time budget increases, which correlates with higher puzzle complexity, the performance curves cross. For time budgets exceeding roughly 3 seconds, the SAT approaches (specifically g42 and OR-Tools) surpass A*.

Results for ILP and CP approaches in the decision problem are shown in the appendix \autoref{app:results|ilp-csp-solvers}. With OR-Tools, the results were very similar, but slightly inferior to the SAT formulation in OR-Tools. With the other solvers (Z3, CPLEX, Gurobi), the ILP and CP formulations were considerably less efficient.

\subsection{Comparison of Solvers on the Minimization Problem}
\label{sec:results|minimizing-solvers}
Here, we compare approaches concerning their performance in the minimization problem. In this case, the minimum number of moves to solve a puzzle is not known to the algorithm. While A* search finds a minimum length solution by design, the other algorithms need some modification. In Linear and Constraint Programming, the objective of minimizing a variable can be declared. Therefore, it is in principle enough to add such a variable and a constraint stating that after that many moves, no more flooding operations are performed, see \autoref{sec:solvers|cp|minimization} for the details. Solving the minimization problem with a SAT-Solver requires repeated calls to the SAT-Solver of the decision problem, narrowing down the number of moves, see \autoref{sec:SAT|minimization}.
\begin{figure}
    \centering
    \includegraphics[width=0.8\linewidth]{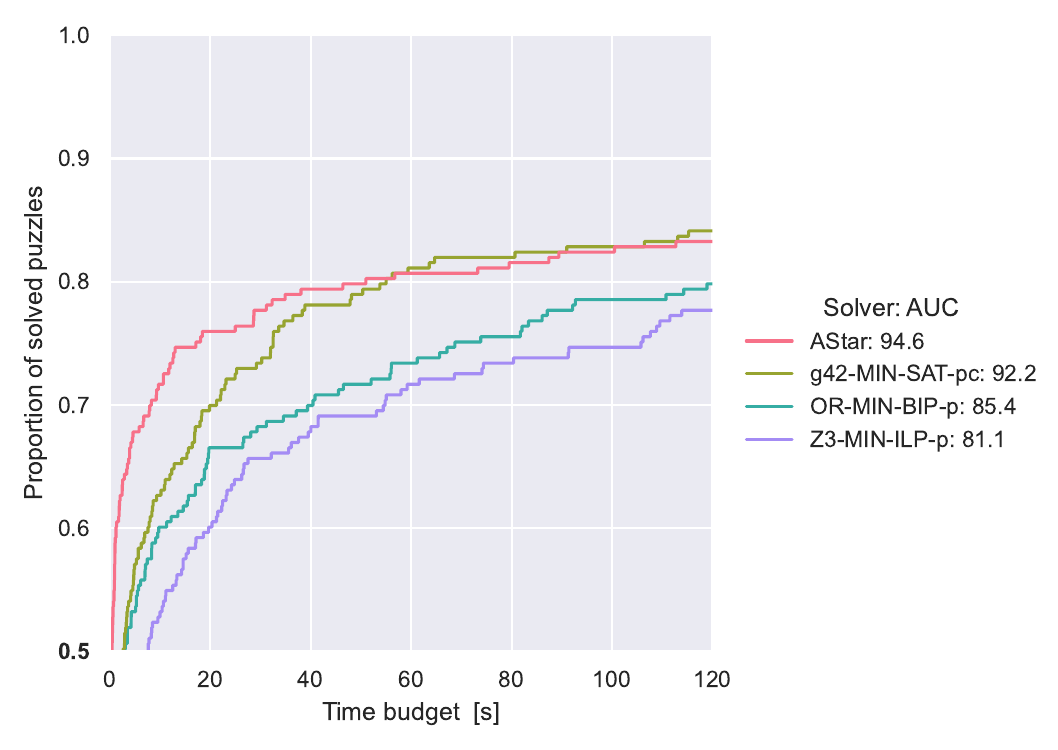}
    \caption{\textbf{Flood-It minimization problem.} Comparison of the solution times for the solvers in the \emph{minimization} problem as survival plot. For each time budget on the $x$-axis up to the timeout limit of 120~s, the curves represent the ratio of puzzles solved within that time limit. Data: 233 instances of the test set, see \autoref{sec:experiments|instances}.}
    \label{fig:comp-mini}
\end{figure}
In the same way as above, \autoref{fig:comp-mini} displays the comparison as a survival plot. The ILP and CP approach performance is, in general, worse. Otherwise, A* search is more successful on `less complex' instances, while the SAT-Solver can solve more puzzles for time budgets larger than about 60~s. Different from the decision problem, within the SAT solvers, g42 performs better than OR-Tools, which is quite surprising as the minimization for both SAT solvers is realized by the same repeated calls of the decision solver. The reason for this surprising change is the different efficiency when the decision is not called with the actual minimal solutions length, but with a lower value, as required in the iterative search. In this setting, OR-Tools performs inferior to g42, see \autoref{app:solving_time_vs_bound}.

For a more detailed comparison of A* and SAT on a puzzle basis, \autoref{fig:mini-scatter} shows a scatter plot. 

Color coding the minimum solution length of each instance reveals that the solution time of A* search is highly correlated with this parameter, while this dependency is weaker for the SAT solver.

\begin{figure}[ht!]
    \centering
    \includegraphics[width=.75\linewidth]{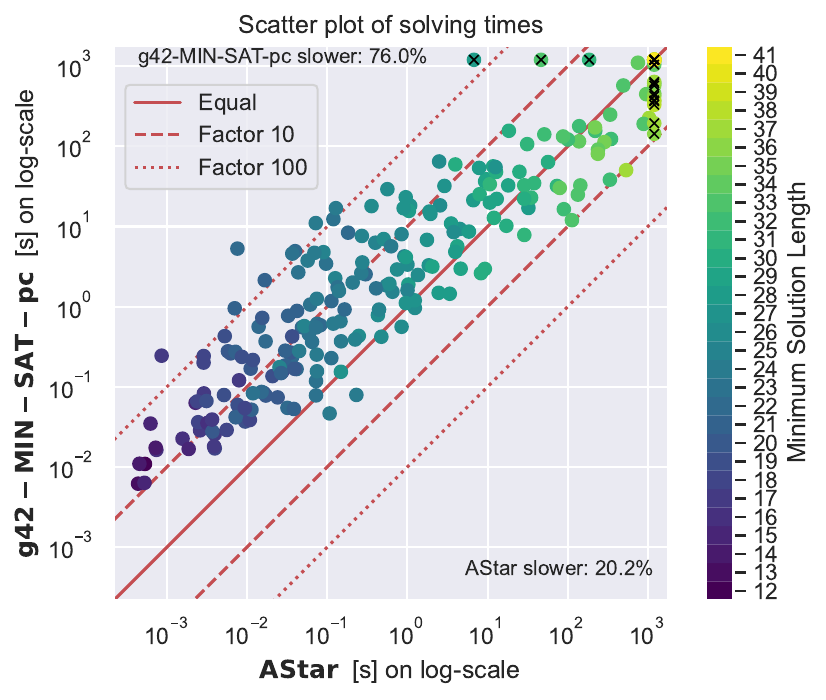}

    \caption{Scatter plot of solution times for the minimization problem using the A* and the SAT-Solver. Each puzzle instance is one data point, and their solution times for A* are on the $x$- and for SAT on the $y$-axis. For points above the diagonal line, A* finds the solution faster, while points below correspond to puzzles that were solved faster with SAT. Puzzle instances are color-coded by the minimum solution length. Dots that are crossed out were not within the time limit of 1200~s by at least one of the solvers. (Note that the survival plot in \autoref{fig:comp-mini} has a different time limit of 120~s.) Data: 233 instances of the test set, see \autoref{sec:experiments|instances}.}
    \label{fig:mini-scatter}
\end{figure}

\section{Discussion}
\label{sec:discussion}
Flood-It puzzles designed for human solvers to find minimum solutions typically have sizes up to 12x12 and use not more than six colors. Puzzle instances in this complexity range are rapidly solved by any of the investigated approaches and implementations. To challenge algorithmic development, we investigated randomly generated Flood-It instances of sizes up to 20x20 using up to 10 colors. To be computationally feasible, this higher level of complexity requires domain-specific optimization in A* search and, for general-purpose solvers, the appropriate choice of solving engines.

The main contribution of this paper is to present an approach using A* search to solve Flood-It puzzle instances. Our approach is endowed with tailored ingredients that are crucial for efficiency. First, we derive an equivalence of game states that allows substantial pruning. Second, we employ a heuristic based on safe and combined moves. Third, we introduce pruning strategies based on strategic considerations. Finally, a good tie breaking strategy is selected.

Additionally, approaches with general solvers (SAT, Integer and Constraint Programming) have been systematically explored. All considered approaches have been evaluated for the decision and for the minimization problem (\emph{existence of a solution of given length} and \emph{finding a shortest solution}) of Flood-It. 

In the following, we first discuss our validation strategy, then the outcome of optimizing each approach and subsequently the comparison across the different approaches.

\subsection{Validation of the Solving Approaches}
\label{sec:discussion|validation}
This work features a systematic and clean performance validation. The best performing variant of each approach was chosen based on a training set of 99 randomly generated puzzles. The final comparison was done on an independently generated test set of 233 puzzles. These puzzles were spanning sizes of side lengths from 12 to 20 and numbers of colors from 5 to 10. For each combination of those, 100 puzzles have been generated and each puzzle featuring a solution length that was not encountered before was collected for the test set.

\subsection{A* search for Flood-It}
\label{sec:discussion|astar}
%
\paragraph{Tie-Breaking}
The employed strategies for breaking ties (selecting a game state among all with the lowest $f$-value) were not specific to Flood-It. In accordance with the literature \cite{asai_tie-breaking_2017} we found tie-breaking in favor of small $h$ values efficient and the effect of third-level tie-breaking (fifo vs.\@ lifo) negligible, see \autoref{fig:astar|ties|survival}. Different from \cite{asai_tie-breaking_2017}, we found $[f, h, \text{fifo}]$ clearly superior to $[f, \text{lifo}]$, see \autoref{sec:results|astar|tiebreaking}. Since their result is based on an evaluation on datasets from 35 different domains, it has more general significance. Note, that we did not implement the more sophisticated diversification strategy proposed in \cite{asai_tie-breaking_2017}.

\paragraph{Equivalence vs.\@ Equality.}
The implementation that treated equivalent game states, see \autoref{theorem:game_states_equal_min}, as one node of the search space clearly outperformed the implementation that treated each discovered game state as a separate node, see \autoref{fig:scatterplot-Astar-equiv-equal}. There has been only one low complexity instance on which the latter had a shorter running time. 
The figure indicates that the effectiveness of equivalence compared to equality generally increases for the test instances that are harder to solve. We thus conclude that the additional computation costs for the equivalence are compensated for by the reduced size of the search space, so its implementation can be an essential optimization in using A* to solve Flood-It instances.

\paragraph{Heuristics.}
The heuristic based on the lower bound $LBP$ (\autoref{algo:lower_bounds_puchert}) proved more efficient than relying solely on the level structure ($LBL$, \autoref{def:lbls}), which in turn outperformed the baseline of using the number of unsolved colors, see \autoref{fig:astar|lb}. We conclude that the increased computational overhead of calculating the tighter bounds is more than compensated for by the resulting reduction in the number of expanded nodes.

\paragraph{Pruning.}
While pruning schemes derived from domain-specific Flood-It strategies show a benefit in basic A* search, the advantage vanishes when other optimizations (exploiting equivalent game states and the efficient $LBP$ heuristic) are used.

\paragraph{Heaps.}
The solving times with A* increased substantially more than with the SAT approach in the puzzles of the highest complexity. We tested A* also with priority queue implementations using Pairing and Fibonacci heaps, which have a better asymptotic running time than the standard binary heaps and can therefore be better suited for applications with large queues. However, in our experiments, the differences of using different heap implementations were negligible.

\subsection{Solving Flood-It with SAT Solvers}
\label{sec:discussion|sat}
Different formulations and optional constraints for the SAT approach have been evaluated. The more complicated formulation proposed in \cite{vanstiphoutSATEncodingsFloodIt2020} performed inferior to the direct formulation in our experiments, \autoref{fig:comp-SATv2}. Additional constraints for symmetry breaking (P) and (C) improved performance in general, see \autoref{fig:comp-SAT-variants}, while additional constraint (E) was detrimental in most cases. Also, it made solvers resort to randomization, leading to a high variance of running time in repeated runs, see \autoref{fig:var-solvers}, and on average to higher running times, see \autoref{fig:var-methods-subplots}.

With respect to the choice of the SAT solver engine, some of the dedicated SAT solvers from PySAT, in particular \solver{g42} (Glucose 4.2.1), performed overall the best. For the decision problem with the correct minimum solution length given, OR-Tools performed better, see \autoref{fig:comp-decision}. However, when the decision solver was called with one less than the minimum solution length, it took considerably more time to find out the infeasibility than \solver{g42}, see \autoref{fig:solving_time_vs_bound_average}. This is even more true for the Z3 solver, see \autoref{fig:solving_time_vs_bound_average|Z3}. Since this probe is required when approaching the minimization problem with a SAT solver, \solver{g42} emerged as the winner.

\subsection{Solving Flood-It with Integer and Constraint Programming}
\label{sec:discussion|programming}
Solving Flood-It with integer programming was implemented using the Gurobi, CPLEX, Z3 and OR-Tools solver. Constraint programming was implemented in CPLEX, Z3 and OR-Tools. Overall, these approaches were not competitive with the SAT formulations and provided no additional performance benefits (results not shown): With OR-Tools, all three approaches yielded essentially the same running time. In Z3, the ILP and CP formulations were substantially inferior to SAT in Z3, with ILP being slightly better than CP. The results obtained with Gurobi and CPLEX were far behind, with Gurobi being superior to CPLEX, and within CPLEX CP being more efficient than ILP.

It is however possible that the default configuration of CPLEX and Gurobi, which we used, is not an appropriate choice for our formulations of Flood-It, and thus they might perform better with different configurations.

\subsection{Comparison of all Approaches}
\label{sec:discussion|comparison}
Concerning the relative performance of A* search, SAT-, and ILP/CP-based solving, we had three hypotheses: (1)~A* is better suited for the minimization problem, while SAT solvers are better for the decision problem. (2)~The SAT, ILP, and CP formulations perform similarly on the decision problem. (3)~ILP and CP are better suited for the minimization problem than SAT due to their intrinsic optimization capability.

Hypothesis (1) was confirmed for the majority of instances. A* solved 75\% of the puzzle instances in the minimization problem faster than the SAT-based approach. While A* efficiently traverses the search space for lower complexities due to the $LBP$ heuristic and game state equivalences, its performance degrades on complex instances because of the exponential growth of the search space. Conversely, the SAT-based approach scales more favorably with complexity, overtaking A* as the minimum solution length increases.

Concerning (2), the SAT, ILP, and CP formulations performed on par in their implementations with OR-Tools. Within Z3 and CPLEX, the differences were more pronounced, but the most significant variance was found between the different solving engines. PySAT and OR-Tools achieved the best performance, outperforming Z3, while Gurobi—and to an even greater extent CPLEX, showed lower performance.

Hypothesis (3) was not supported by the data. Even within the same solving engine (e.g., OR-Tools), an iterative approach using the SAT formulation was more efficient than using the native minimization within the ILP or CP formulations. A potential reason for this outcome is that the minimization formulation requires a larger variable array that encompasses the maximum assumed solution length. Using our solution length prediction, this overhead is restricted to about five moves. However, searching for a solution that has 30 moves in a variable space prepared for 35 moves implies a large amount of suboptimal combinations among which the shortest must be found.

\subsection{Prediction of the Solution Length}
\label{sec:discussion|predition_sol_len}
While a good prediction of the minimum solution length was obtained by a linear regression from \emph{size} and \emph{colors}, there is also evidence that the actual relationship is nonlinear, cf. \autoref{fig:sol-len-stats}. For the complexity range used in this study, the small inaccuracy involved with a \emph{linear} regression was accepted. In applications that would involve a larger range of complexities, exploring nonlinear predictions seems advisable.

\section{Conclusion}
\label{sec:conclusion}
For the efficiency of A*, several domain-specific optimizations have been essential, such as game state equivalences, and lower bounds, while pruning by move strategies only yielded marginal gains. With the integration of those, A* was overall the fastest solver not only for the minimization but also for the decision problem. Since A* is a guided search towards the optimal solution, its superiority in the minimization problem came as no surprise. But A* does not scale well with large search spaces, which is reflected here by the finding that the SAT approach dominates with increasing complexity, see \autoref{fig:comp-mini} and \ref{fig:mini-scatter}. For the decision problem, this turning point is found at a much lower complexity, see \autoref{fig:comp-dec|testset}.

The SAT, ILP and CP formulations are logically equivalent. This was reflected by the fact that their implementations in OR-Tools all had essentially the same running time. With different solvers, however, the performance was found to be largely different. Apart from OR-Tools, using a dedicated SAT solver yielded the best performance and was clearly superior to implementations using Z3, Gurobi and CPLEX. 
While OR-Tools had an edge by a narrow margin over the dedicated SAT solver in the decision problem, interestingly, this did not extend to the minimization problem, in which \solver{g42} was superior. Moreover, there was a surprise in the results for the minimization problem. In the decision problem, the performance of the SAT, ILP and CP formulations were on par in their OR-Tools implementation. Therefore, we hypothesized that ILP and CP would win in the minimization problem, as they allow to canonically integrate the minimization, while a SAT formulation requires iterative search. However, the opposite was the case, see \autoref{fig:comp-mini}. A potential reason for this outcome is that the minimization formulation requires a larger variable array that encompasses the largest assumed solution length. Using our solution length prediction, this overhead is restricted to about five moves. However, searching for a solution that has 30 moves in variables prepared to represent 35 moves implies a large number of suboptimal solutions, among which a shortest has to be found.

\section*{Acknowledgement}
Python code was written in PyCharm by JetBrains. In particular, the scripts for generating the figures and tables of the results were aided by JetBrains AI using the LLM gemini-pro-1.5 and openai-gpt-4o. This paper was written using \LaTeX\ in Overleaf.

\FloatBarrier

\clearpage
\bibliographystyle{apalike}
\bibliography{references}

\newpage
\appendix

\section{Appendix}
\label{sec:app}
%


\subsection{Alternative SAT Formulation}
\label{app:solvers|satv2}
While the SAT formulation presented in \autoref{sec:solvers|specific} is a straightforward translation of the rules of the game, we also consider an alternative approach proposed in \cite{vanstiphoutFlooditSatProblem2020}. There an additional variable $\TchFld(t,v)$ is introduced for all time instances $t$ and all vertices $v$ which is true, iff $v$ is adjacent to a vertex that is flooded at time $t$ (`touching the flood'). Equations in gray are the same as in the variant in \autoref{sec:solvers}. Again, conditions with letters as equation number are optional.

\setcounter{equation}{0}
\begin{align}
    \mathcolor{gray}{& \quad \Fld(0, 0)
        \label{SAT2|pivot}}\\
    \mathcolor{gray}{\forall_{0<v<V}
        & \quad \neg\Fld(0, f)
        \label{SAT2|notfloodedat0}}\\
    \mathcolor{gray}{\forall_{v<V}
        & \quad \Fld(T,v)
        \label{SAT2|finallyallflooded}}\\
    \mathcolor{gray}{\forall_{t<T}\forall_{c<C-1}\forall_{c<d<C}
        & \quad \neg\Move(t,c) \vee \neg\Move(t,d)
        \label{SAT2|maxonemove}}\\
    \mathcolor{gray}{\forall_{t<T}\forall_{v<V}
        & \quad \neg\Fld(t+1,v) \vee \Fld(t,v) \vee \Move(t,color[v])
        \label{SAT2|floodingrequiresmove}}\\
    \forall_{t<T}\forall_{v<V}
        & \quad \neg\Fld(t+1,v) \vee \Fld(t,v) \vee \TchFld(t,v)
        \label{SAT2|floodingrequirestouching}\\
    \forall_{t<T}\forall_{v<V}
        & \quad \neg\TchFld(t,v) \vee \bigvee_{u\in adj[v]} \Fld(t,u)
        \label{SAT2|touchingrequiresfloodedneighbor}\\ 
    \forall_{t<T}\forall_{v<V}\forall_{u\in adj[v]}
        & \quad \neg\Fld(t,u) \vee \TchFld(t,v)
        \label{SAT2|floodingrequirestouchedneighbor}\\
    \mathcolor{gray}{\forall_{t<T}
        & \quad \bigvee_{c<C} \Move(t, c)
        \label{SAT2|minonecolor}} \tag{C}\\
    \mathcolor{gray}{\forall_{t<T}\forall_{v<V}
        & \quad \neg\Fld(t,v) \vee \Fld(t+1,v)
        \label{SAT2|floodedforever}} \tag{P}\\
    \forall_{t<T}\forall_{v<V}
        & \quad \neg\Move(t,color[v]) \vee \neg\TchFld(t,v) \vee \Fld(t+1,v)
        \label{SAT2|movefloodstouching} \tag{E}
\end{align}

\FloatBarrier
\subsection{Comparison of Different SAT Solvers}
\label{app:results|sat-solvers}
The PySAT toolkit \cite{ignatiev2018pysat} includes a number of SAT solvers. \autoref{fig:comp-SAT} compares the performance of seven SAT solvers\footnote{%
\textbf{g42}: Glucose 4.2.1\cite{audemard2018glucose}; 
\textbf{cd19}: CaDiCaL 1.9.5\cite{biere2017cadical};
\textbf{gc4}: Gluecard 4.1 (based on Glucose, see \cite{pysatdoc});
\textbf{mgh}: MinisatGH\cite{een2003minisat};
\textbf{mcm}: MapleCM\cite{liang2016maplesat};
\textbf{cms}: CryptoMinisat\cite{soos2016cryptominisat}
}
using the SAT formulation of \autoref{sec:solvers} with constraints (P) and (C). Comparing the running time on a number of tasks with a high variance in complexity is not straightforward. Most ways of displaying would be completely dominated by the complex tasks. We plot the proportion of puzzles (on the $y$-axis) that a method can solve within a time budget (on the $x$-axis), see \autoref{fig:comp-SAT}. As the curve of the Glucose 4.2.1 Solver (\solver{g42}) is on top for most of the time, we use this solver for all subsequent investigations, unless stated otherwise.

\begin{figure}[ht!]
    \centering
    \includegraphics[width=\linewidth]{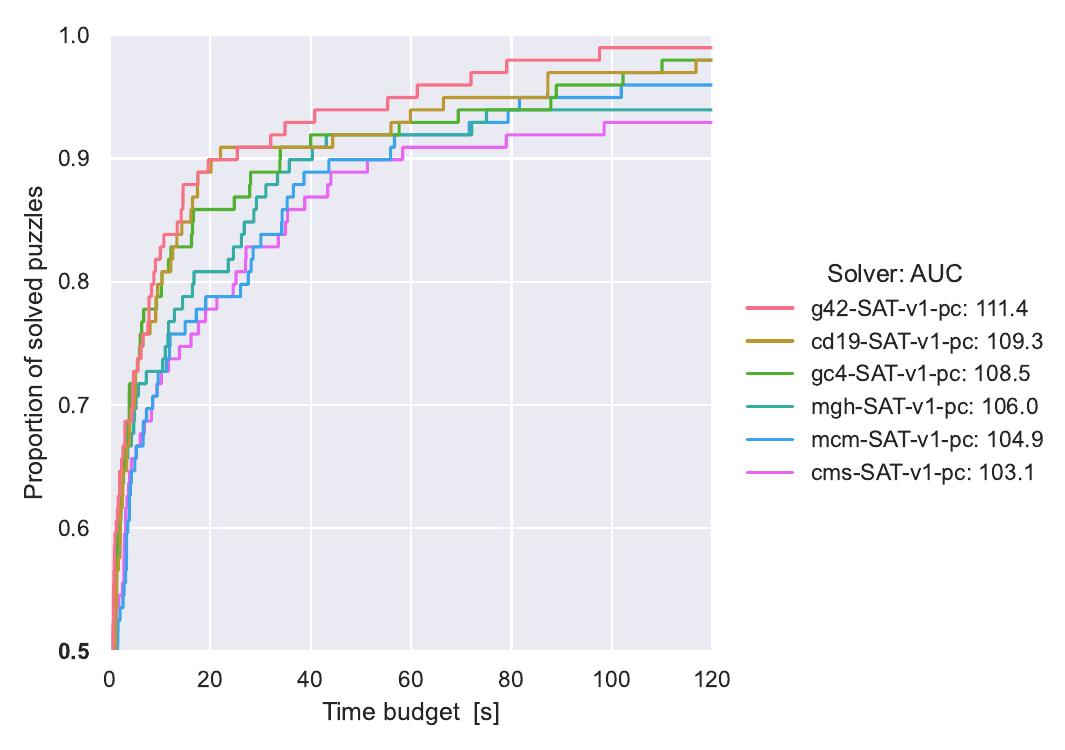}
    \caption{The graph compares the overall performance of the different SAT solvers available in PySAT on 99 Flood-It instances as a survival plot. For the time budget on the $x$-axis up to two minutes, the curves show the proportion of the puzzles solved. As an overall performance measure, we use the area under the curve (AUC). The values for each solver are given in the legend. For better display, only the range 0.5 to 1 is shown on the $y$-axis. The abbreviations for the solver are: \textbf{g42}: Glucose 4.2.1, \textbf{cd19}: CaDiCaL 1.9.5, \textbf{gc4}: Gluecard 4.1, \textbf{mgh}: MinisatGH, \textbf{mcm}: MapleCM, \textbf{cms}: CryptoMinisat; see the main text for references. Moreover, `v1' refers to the first SAT formulation given in \autoref{sec:solvers} and `pc' indicates that the optional constraints (P) and (C) have been used.}
    \label{fig:comp-SAT}
\end{figure}

In the same way as for different formulations of the SAT problem in \autoref{fig:var-methods}, we investigated variability across the different SAT solvers that were introduced in \autoref{fig:comp-SAT}. We restricted this analysis to the SAT formulation `pc', which showed the best performance in \autoref{sec:results|sat-formulations}. The variability across repeated runs can be quantified as the coefficient of variation (CV), which is the standard deviation divided by the mean. This quantity is independent of scale and can therefore be used across problems of different complexity.

We conducted 10 runs for each solver and each puzzle of the training set, and calculated the CV across the 10 runs. \autoref{fig:var-solvers} shows the result as boxplot with an overlay of the individual data points. Above each distribution, the CV averaged across all puzzles is shown. The individual CVs are large and spread for \solver{cd19} and \solver{mcm}. The values have much less scatter for the other solver, but \solver{mgh} has one and \solver{g42} has five outliers. Also, the median CV for \solver{g42} is higher than for \solver{gc4}, \solver{mgh} and \solver{cms}, but lower than for \solver{cd19} and \solver{mcm}.
\begin{figure}[ht!]
    \centering
    \includegraphics[width=\linewidth]{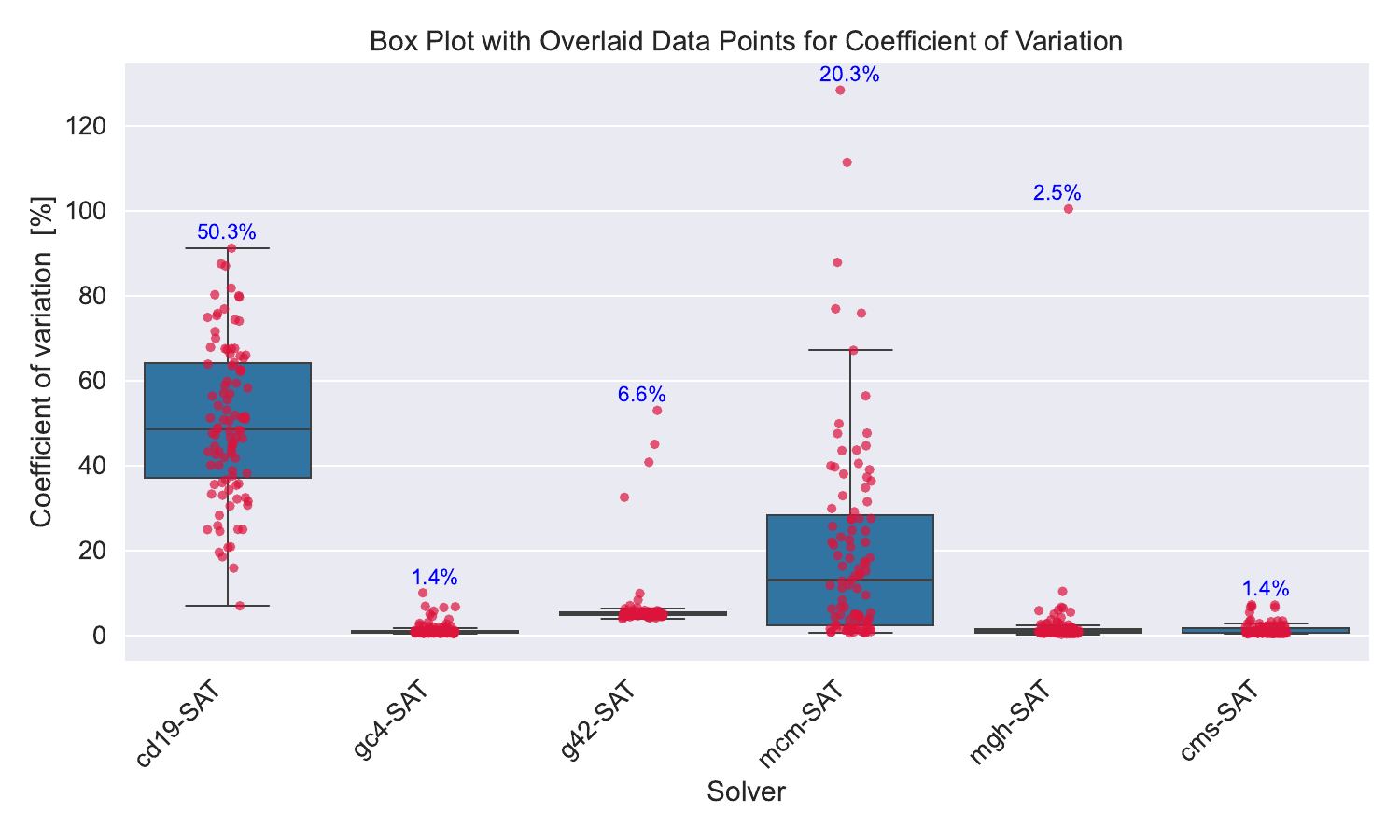}
    \caption{The coefficient of variation (CV) was calculated across 10 runs for each solver and each puzzle of the training set. The results for six different SAT solvers are shown as box plots with an overlay of individual data points.}
    \label{fig:var-solvers}
\end{figure}

\FloatBarrier
\subsection{Comparison of the SAT Formulations}
\label{app:results|SATv2}
\autoref{fig:comp-SATv2} shows the results for the alternative formulation of the SAT problem proposed in \cite{vanstiphoutFlooditSatProblem2020}, see \autoref{app:solvers|satv2}. All variants are slightly slower compared to the first formulation, which is described in \autoref{sec:solvers}, see \autoref{fig:comp-SAT-variants}.
\begin{figure}[ht!]
    \centering
    \begin{subfigure}{0.49\textwidth}
        \includegraphics[width=\linewidth]{survival-SAT-dec-g42.pdf}
        \caption{First SAT formulation}
        \label{fig:comp-SATvariants}
    \end{subfigure}
    \hfill
    \begin{subfigure}{0.49\textwidth}
        \includegraphics[width=\linewidth]{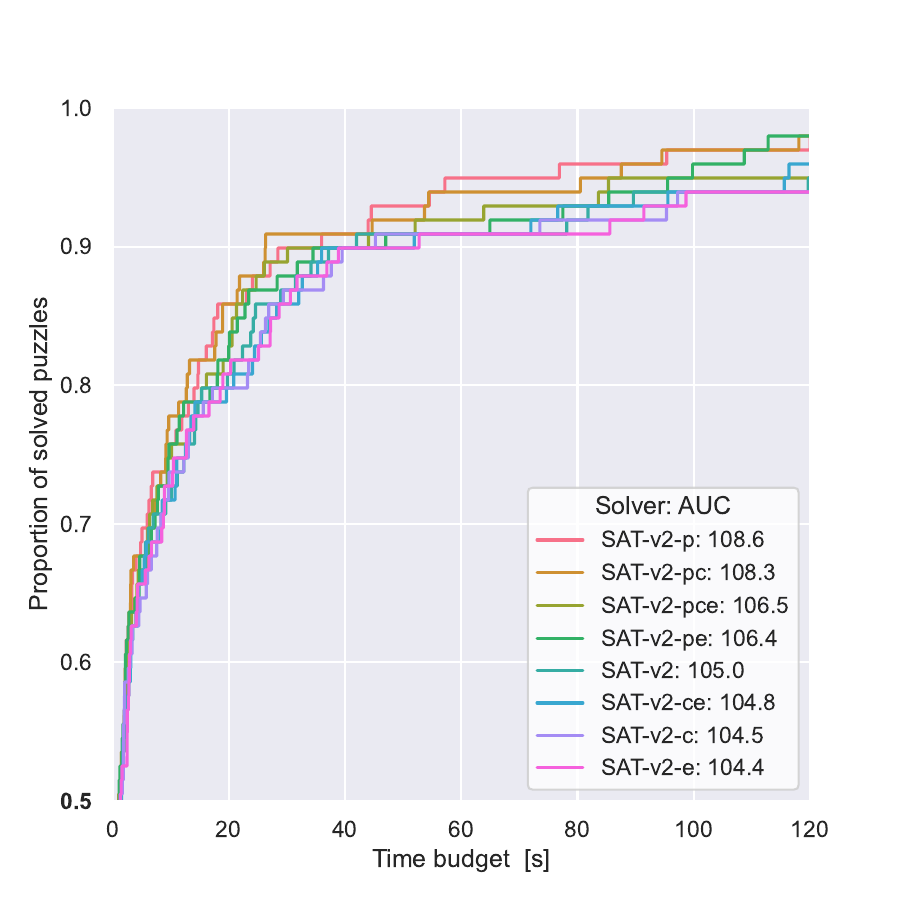}
        \caption{Second SAT formulation}
    \end{subfigure}
    \caption{This graph shows the performance of the two SAT formulations for the decision problem side-by-side. The first formulation is described in \autoref{sec:solvers} and the second in \autoref{app:solvers|satv2}.}
    \label{fig:comp-SATv2}
\end{figure}

\newpage
\subsection{Variation in Solving Time}
\label{app:results|variation}
The time it takes to solve a Flood-It instance varies for different reasons. On one hand, the complexity varies, and on the other hand, some approaches use randomization.

While A* search is deterministic, some SAT, ILP and CP Solver may resort to randomized search in certain situations.
\autoref{fig:var-methods-subplots} shows solution times for five selected puzzle instances in ten repeated runs. The puzzles were solved with all eight variants of the SAT formulation using the \solver{g42} solver. The variation is high whenever constraint (E) is used.
\begin{figure}[ht!]
    \centering
    \includegraphics[width=\linewidth]{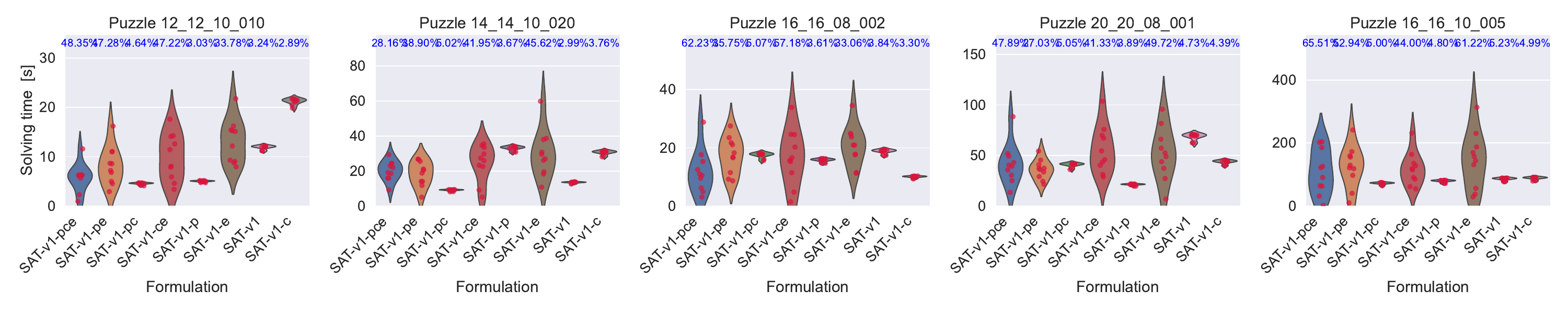}
    \caption{\textbf{Distribution of solving times.} Five selected instances have been solved repeatedly in ten runs using the eight different SAT formulations. The distribution of the ten solving times is shown as violinplot for each solver and instance. The numbers on top are the coefficients of variation.}
    \label{fig:var-methods-subplots}
\end{figure}

While \autoref{fig:var-methods-subplots} shows the CV for the distributions of five selected instances printed above each violin, \autoref{fig:var-methods} shows the \emph{distribution of CVs} to visualize the overall variability in the whole database. This confirms that constraint (E) introduces a high variability in the solving times.
\begin{figure}[ht!]
    \centering
    \includegraphics[width=\linewidth]{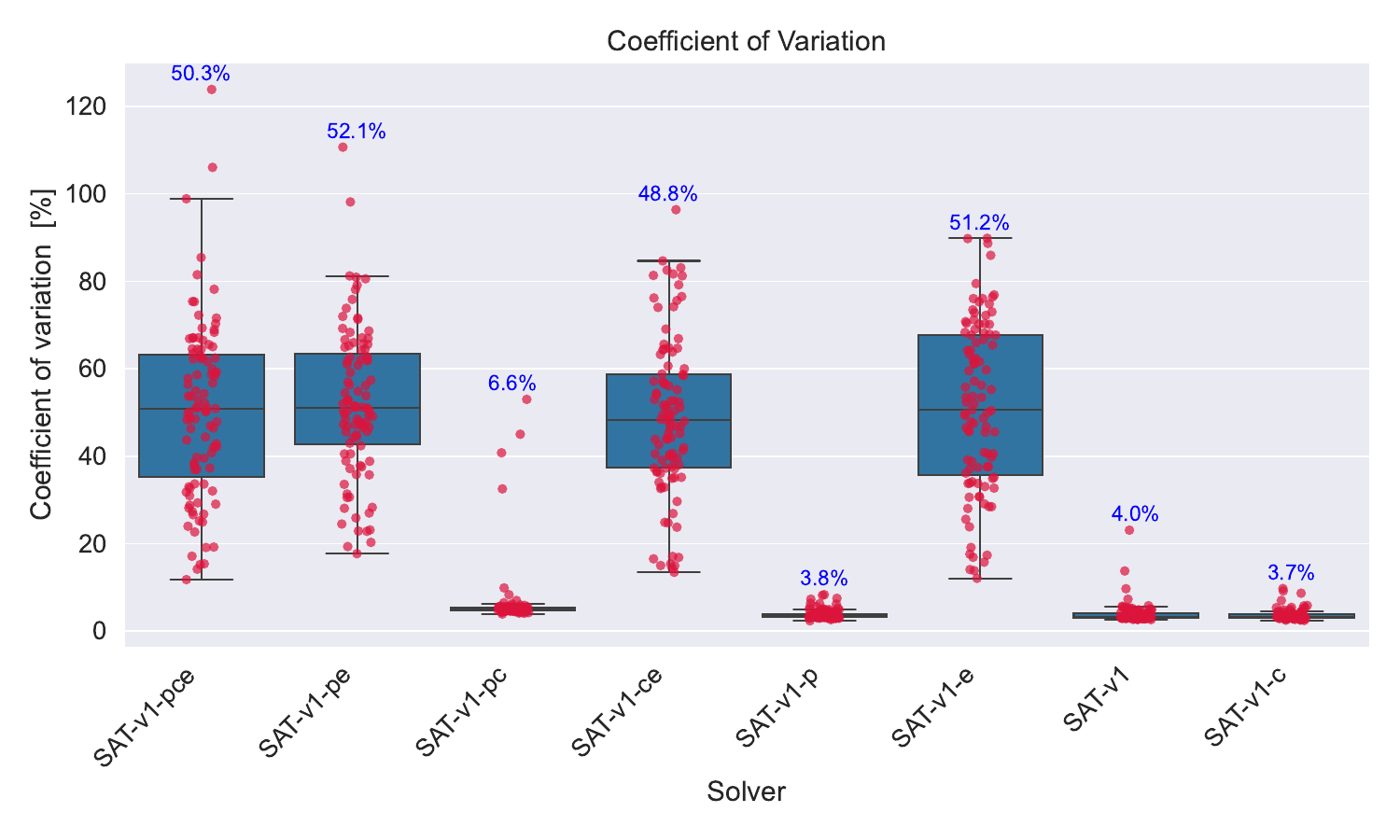}
    \caption{\textbf{Coefficient of variation (CV) across repeated runs.} The same puzzle instances have been solved using different formulations in the SAT solver to calculate CVs. The boxplot of CVs is shown with an overlay of the individual CVs for each instance. A number above a box shows the median CV. Data: 99 Flood-It instances of the training set.}
    \label{fig:var-methods}
\end{figure}

Variability in solving time of the different SAT solvers is shown in \autoref{app:results|sat-solvers}, \autoref{fig:var-solvers}.

While the grid size and the number of colors have a substantial impact on the time required for solving, there is still quite a variability across puzzles, when those parameters are kept fixed, see \autoref{app:var-within-complexity}, \autoref{fig:var-within-complexity}.

\FloatBarrier
\newpage
\subsection{Comparison of Integer and Constraint Programming Solvers}
\label{app:results|ilp-csp-solvers}
\begin{figure}[!ht]
    \centering
    \includegraphics[width=0.66\linewidth]{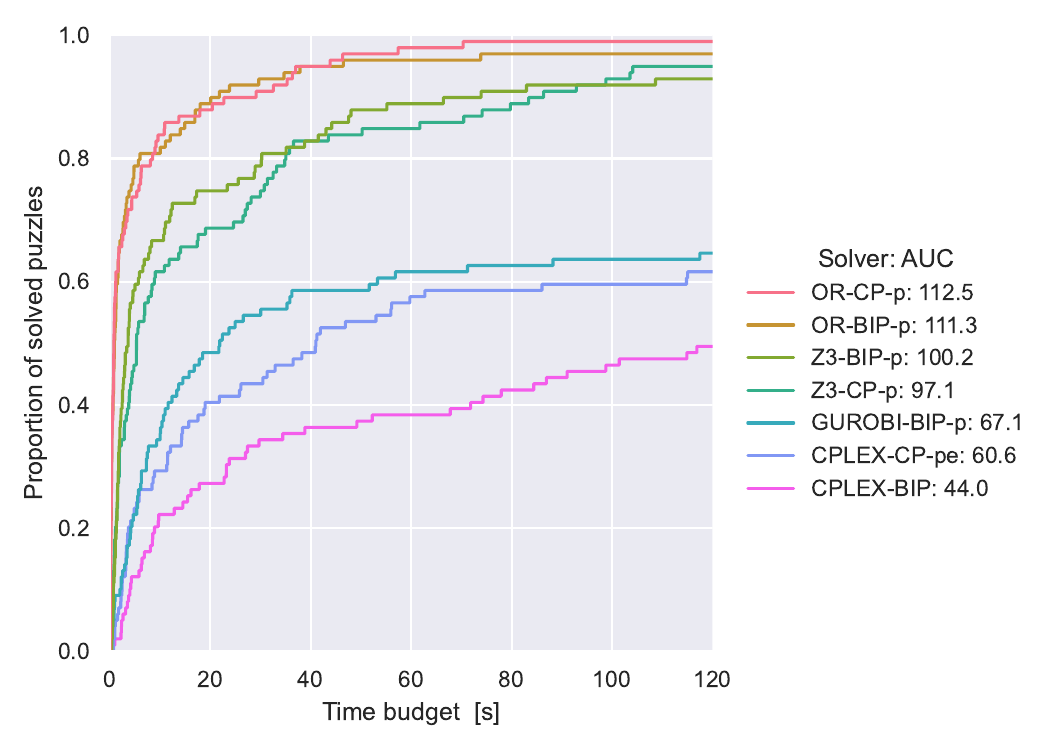}
    \caption{Comparison of the best constraint (CP) and binary integer programming (BIP) version of each solver (Z3, Gurobi, CPLEX) for the decision problem as survival plot. For each time budget on the $x$-axis up to the timeout limit of 120~s, the curves represent the ratio of puzzles solved within that time limit. The values in the legend specify the area under the curve (AUC) as an overall measure of performance. The appendix `-p' refers to the inclusion of the optional constraint (P), see \autoref{sec:solvers}.}
    \label{fig:survival-ILP-CP}
\end{figure}

\FloatBarrier
\newpage
\subsection{Variation in Solving Time for Fixed Size and Colors}
\label{app:var-within-complexity}
\begin{figure}[!ht]
    \centering
    \includegraphics[width=\linewidth]{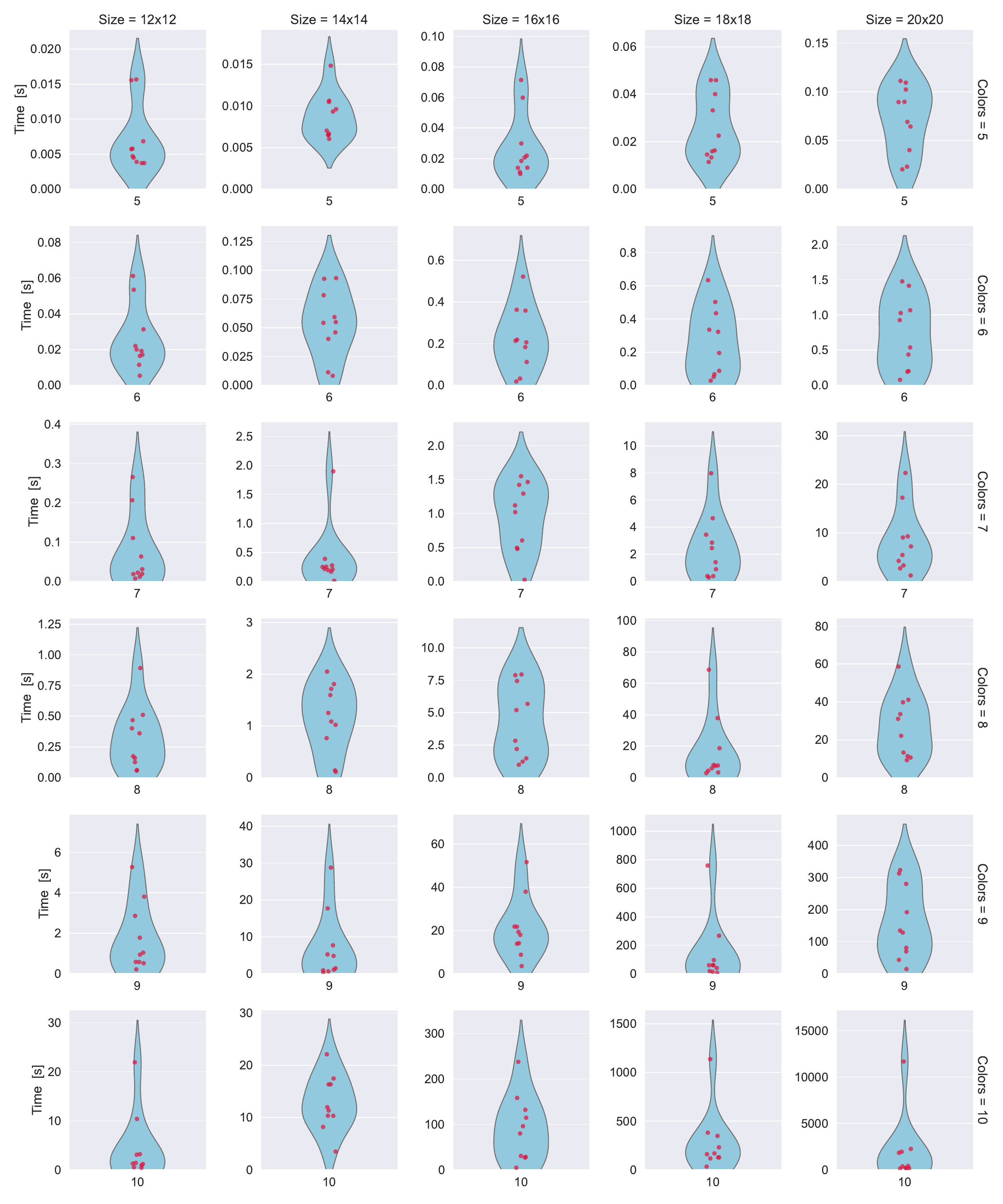}
    \caption{\textbf{Variation in solving time for fixed `complexity'.} Each plot shows data for a fixed combination of puzzle \emph{size} and number of \emph{colors}, which can be seen as a measure of (potential) complexity of the puzzle. The violin plots show the distribution of the solution times for the decision problem.}
    \label{fig:var-within-complexity}
\end{figure}

\FloatBarrier
\newpage
\subsection{Solving Time for SAT Decision Solvers with Inaccurate Bound}
\label{app:solving_time_vs_bound}
Adapting the SAT approach for the minimization problem requires an iterative search, which includes calling the SAT decision solver for values of $T$ that do not match the minimum solution length. For the design of a search strategy it is important to know how the solving times of the SAT decision solver varies, when the given number of moves $T$ is either too low (leading to a failure) or too high (allowing non-optimal solutions). We investigated this dependency exemplarily for ten Flood-It instances using the \solver{g42} solver, see \autoref{fig:solving_time_vs_bound_subplots}. Denoting the minimum solution length by $\Tmin$, we observe that the peak solving duration is in all these examples within the interval $[\Tmin-1, \Tmin+1]$. Solving times for values of $T$ outside this interval are mostly lower than for $T=\Tmin$. Solving times are typically falling off quickly further outside that interval. \autoref{fig:solving_time_vs_bound_average|g42} shows the average solving times across the selected 10 puzzle instances.
\begin{figure}[ht!]
    \centering
    \includegraphics[width=\linewidth]{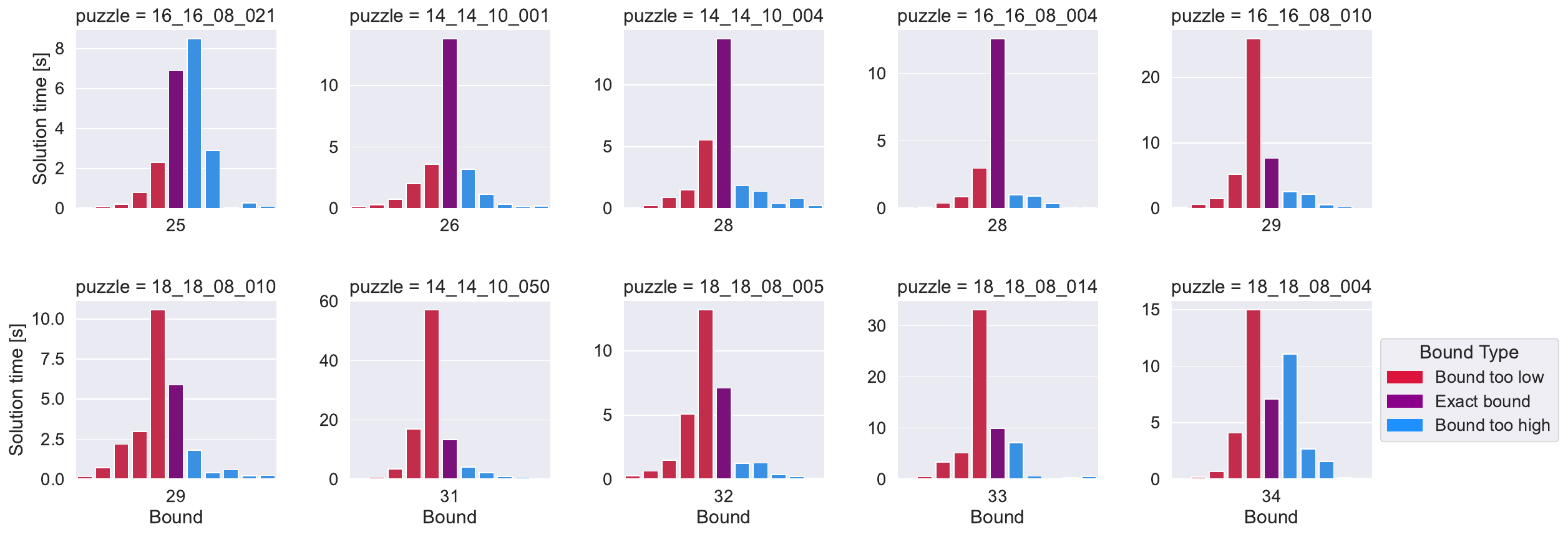}
    \caption{\textbf{Solving time when using different values for solution length $T$.} For ten Flood-It instances, the solution time with the SAT decision solver was measured for solution lengths $T = \Tmin-5,\dots,\Tmin+5$, where $\Tmin$ is the minimum solution length. (Just to clarify: The solvers return `\emph{unsolvable}' for $T<\Tmin$ and might return a suboptimal solution when $T>\Tmin$.)}
    \label{fig:solving_time_vs_bound_subplots}
\end{figure}

\begin{figure}[ht!]
    \centering
    \begin{subfigure}{0.32\textwidth}
        \includegraphics[width=\linewidth]{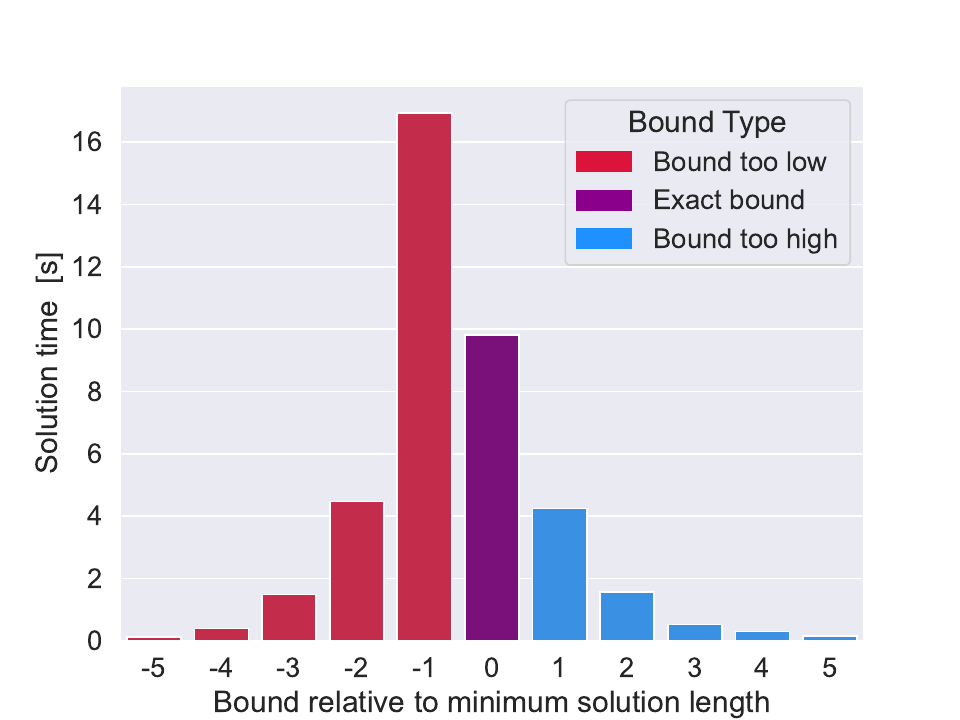}
        \caption{Results using \solver{g42}.}
        \label{fig:solving_time_vs_bound_average|g42}
    \end{subfigure}
    \hfill
    \begin{subfigure}{0.32\textwidth}
        \includegraphics[width=\linewidth]{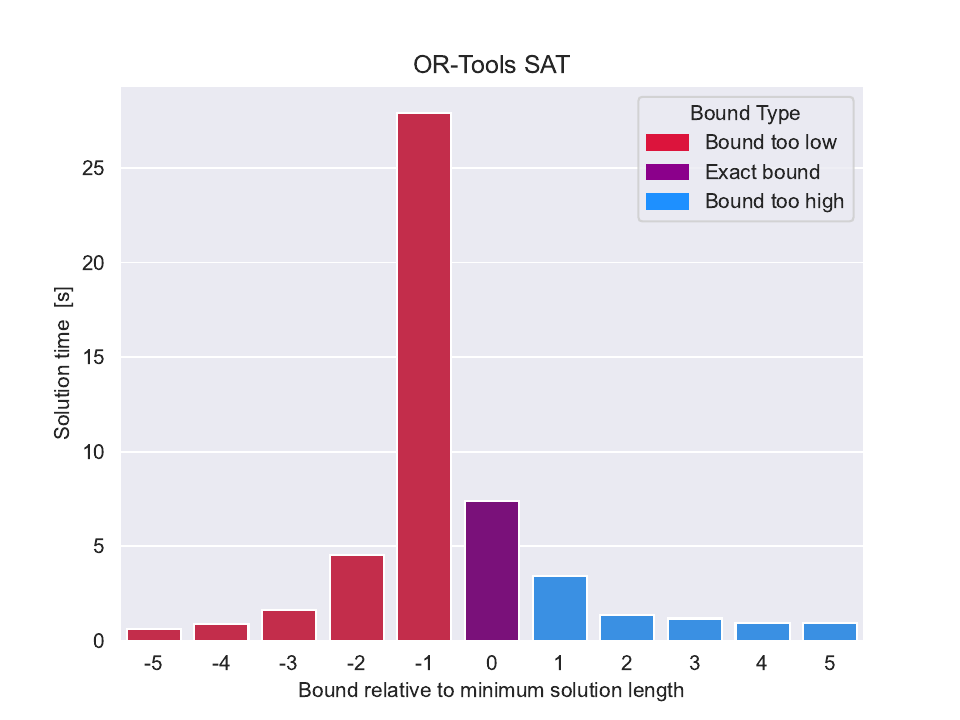}
        \caption{Results using OR-Tools.}
        \label{fig:solving_time_vs_bound_average|ORTools}
    \end{subfigure}
    \hfill
    \begin{subfigure}{0.32\textwidth}
        \includegraphics[width=\linewidth]{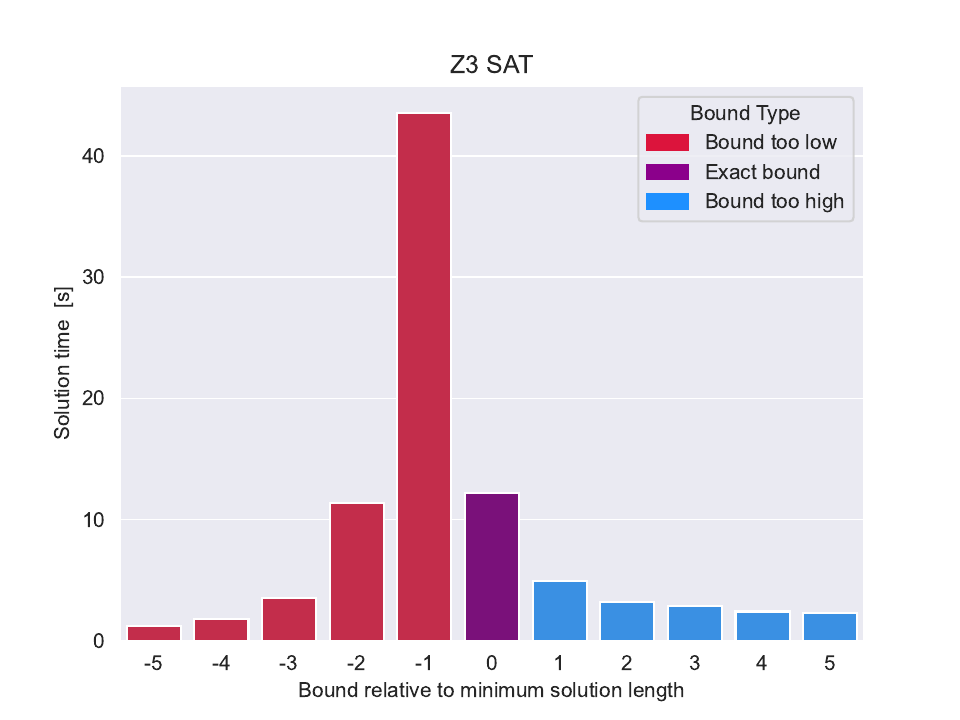}
        \caption{Results using Z3.}
        \label{fig:solving_time_vs_bound_average|Z3}
    \end{subfigure}
    \caption{\textbf{Average solving time when using different values for solution length $T$.} The graph shows the solving times of \autoref{fig:solving_time_vs_bound_subplots} averaged across the 10 puzzle instances. Note the different scaling in the subfigures.}
    \label{fig:solving_time_vs_bound_average}
\end{figure}

Interestingly, the situation is crucially different when using OR-Tools as solver. While this solver is faster than \solver{g42} when the exact bound is given, it is considerably slower when the bound is one below to return the infeasibility of the problem. As this value needs to be probed in the minimization problem, see \autoref{sec:results|minimizing-solvers}, OR-Tools loses in that problem.

\FloatBarrier
\newpage
\subsection{Regression of Solution Lengths in Projections}
\label{app:regr_sol_projection}
\begin{figure}[ht!]
    \centering
    \includegraphics[width=\linewidth]{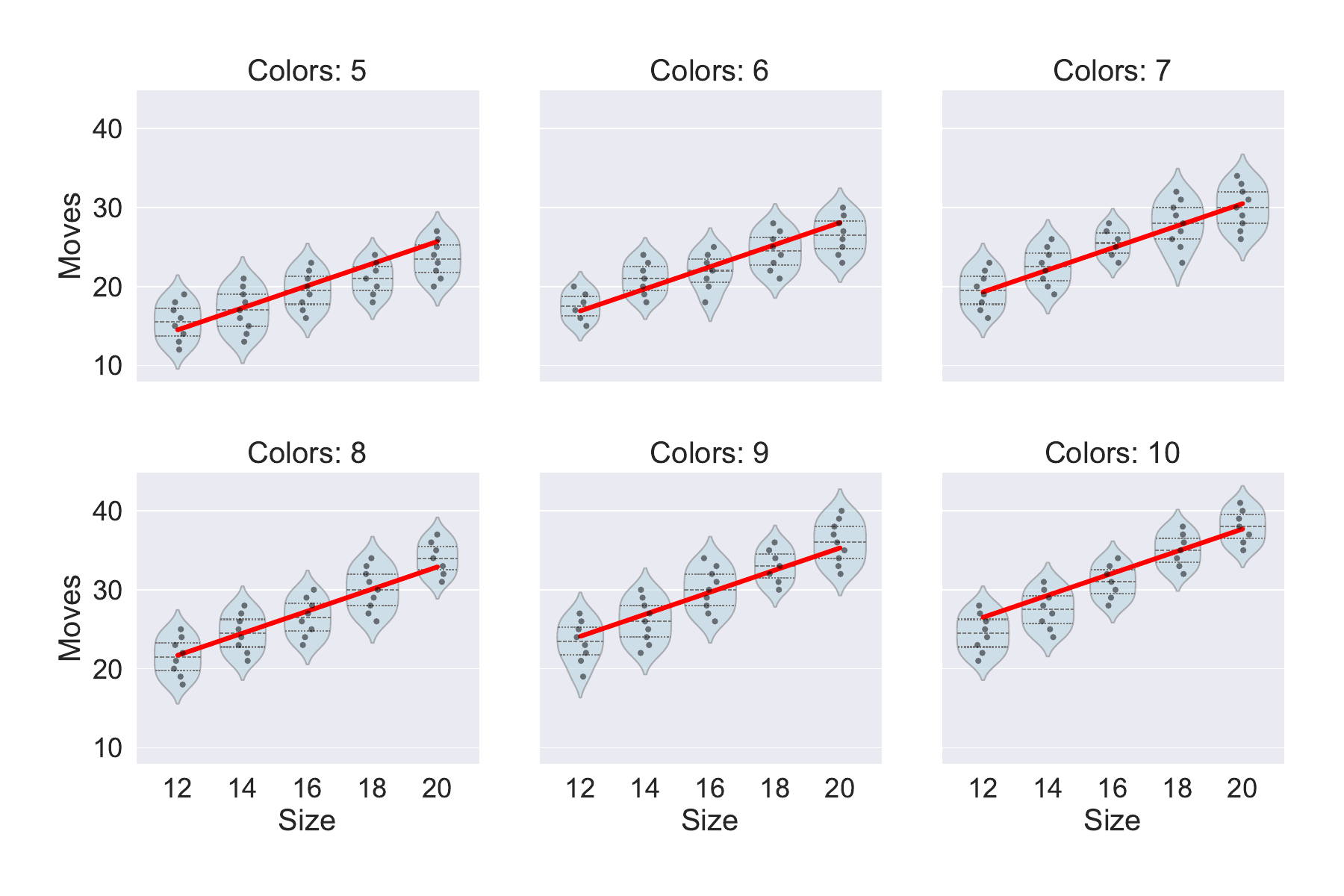}
    \caption{\textbf{Regression of solution length projected to colors.} The plots are projections of the data from \autoref{fig:sol-len-reg} to the different numbers of colors. Each plot shows the distribution of the solution length for the indicated puzzle size as violin plots with an overlay of individual data points. Horizontal lines are shown at quartiles and median. The red line is a projection of the regression plane, see \autoref{fig:sol-len-reg}.}
    \label{fig:sol-len-reg-size}
\end{figure}

\begin{figure}[ht!]
    \centering
    \includegraphics[width=\linewidth]{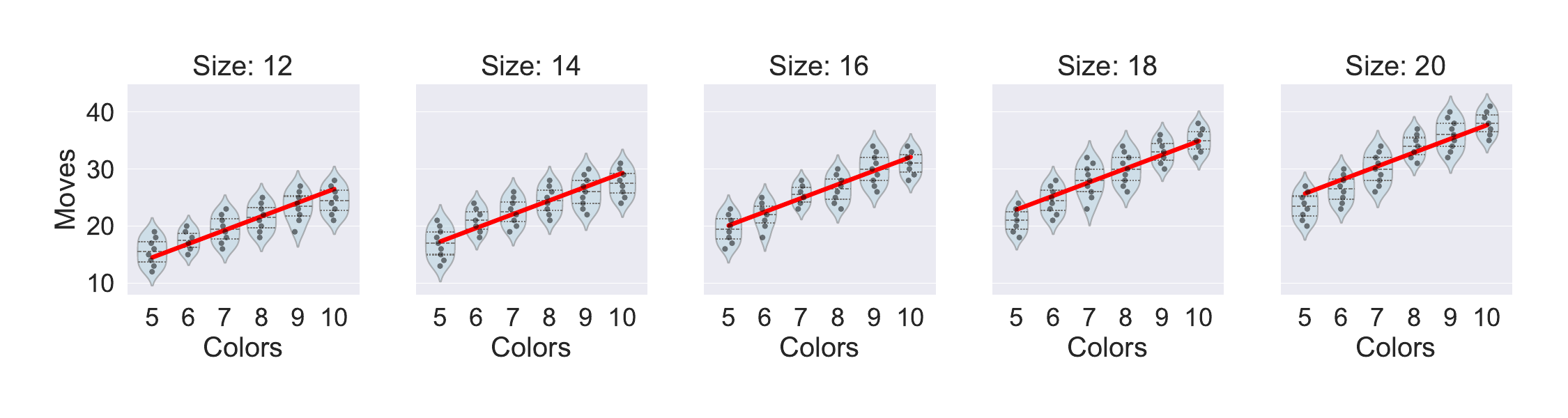}
    \caption{\textbf{Regression of solution length projected to size.} The plots are projections of the data from \autoref{fig:sol-len-reg} to the different puzzle sizes. Each plot shows the distribution of the solution length for the indicated number of colors as violin plots with an overlay of individual data points. Horizontal lines are shown at quartiles and median. The red line is a projection of the regression plane, see \autoref{fig:sol-len-reg}.}
    \label{fig:sol-len-reg-colors}
\end{figure}

\FloatBarrier
\newpage
\subsection{Lower Bound LBL is not a Consistent Heuristic}
\label{app:lbl_not_consistent}
As the example in \autoref{fig:lbl_not_consistent} shows, the lower bound $LBL(S)$ defined in \autoref{def:lbls} does not give rise to a consistent heuristic. Since $LBL(S)$ is admissible by \autoref{lem:lbfurthest}, it can be used in A* to find optimal solutions. However, due to the missing consistency, the search procedure is, in general, not greedy, which means that the same nodes may be visited multiple times, which can hamper efficiency.

\begin{figure}[ht]
    \centering
    \resizebox{0.6\textwidth}{!}{
        	\newcommand{\convexpath}[2]{
		[   
		create hullnodes/.code={
			\global\edef\namelist{#1}
			\foreach [count=\counter] \nodename in \namelist {
				\global\edef\numberofnodes{\counter}
				\node at (\nodename) [draw=none,name=hullnode\counter] {};
			}
			\node at (hullnode\numberofnodes) [name=hullnode0,draw=none] {};
			\pgfmathtruncatemacro\lastnumber{\numberofnodes+1}
			\node at (hullnode1) [name=hullnode\lastnumber,draw=none] {};
		},
		create hullnodes
		]
		($(hullnode1)!#2!-90:(hullnode0)$)
		\foreach [
		evaluate=\currentnode as \previousnode using \currentnode-1,
		evaluate=\currentnode as \nextnode using \currentnode+1
		] \currentnode in {1,...,\numberofnodes} {
			let
			\p1 = ($(hullnode\currentnode)!#2!-90:(hullnode\previousnode)$),
			\p2 = ($(hullnode\currentnode)!#2!90:(hullnode\nextnode)$),
			\p3 = ($(\p1) - (hullnode\currentnode)$),
			\n1 = {atan2(\y3,\x3)},
			\p4 = ($(\p2) - (hullnode\currentnode)$),
			\n2 = {atan2(\y4,\x4)},
			\n{delta} = {-Mod(\n1-\n2,360)}
			in 
			{-- (\p1) arc[start angle=\n1, delta angle=\n{delta}, radius=#2] -- (\p2)}
		}
		-- cycle
	}
	
	\begin{tikzpicture} [every node/.style={fill opacity=0.5, text opacity=100},
		arrow/.style={-latex, shorten >=1ex, shorten <=1ex},
		graphnode/.style={fill opacity=0.5, text opacity=100},
		innernode/.style={draw, draw opacity=100, circle, fill opacity=1, text opacity=100, scale=0.7},
		rednode/.style={innernode, fill=red, node contents={1}},
		cyannode/.style={innernode, fill=cyan, node contents={2}},
		greennode/.style={innernode, fill=green, node contents={3}},
		yellownode/.style={innernode, fill=yellow, node contents={4}},
		magentanode/.style={innernode, fill=magenta, node contents={5}},
		orangenode/.style={innernode, fill=orange, node contents={6}},
		edgenode/.style={draw opacity=0, scale=0.8, midway, above},
		rededgenode/.style={rednode, edgenode},
		orangeedgenode/.style={orangenode, edgenode},
		cyanedgenode/.style={cyannode, edgenode},
		greenedgenode/.style={greennode, edgenode},
		blueedgenode/.style={yellownode, edgenode},
		magentaedgenode/.style={magentanode, edgenode}]
		
		\node (111)  at (0, 1)
		{
			\begin{tikzpicture}[]
				\node (1) at (0,2) [greennode];
                
				\node (2) at (1,3) [rednode];
				\node (3) at (1,1) [cyannode];
                
				\node (4) at (2,3.5) [orangenode];
				\node (5) at (2,2.5) [greennode];
				\node (6) at (2,1.5) [yellownode];
                    \node (7) at (2,0.5) [magentanode];
				
				\draw (1) -- (2) -- (4);
                    \draw (2) -- (5);
                    \draw (1) -- (3) -- (6);
                    \draw (3) -- (7);

				\node (11) at (0,4)[opacity=1] {$L_0$};
				\node (11d) at (0,0.3)[opacity=0] {};
				
				\node (12) at (1,4)[opacity=1] {$L_1$};
				\node (12d) at (1,0.3)[opacity=0] {};
				
				\node (13) at (2,4)[opacity=1] {$L_2$};
				\node (13d) at (2,0.3)[opacity=0] {};

				\begin{scope}[on background layer]
					\fill[yellow, opacity=0.3] \convexpath{11, 11d}{11pt};
					\fill[yellow, opacity=0.3] \convexpath{12, 12d}{11pt};
					\fill[yellow, opacity=0.3] \convexpath{13, 13d}{11pt};
				\end{scope}
			\end{tikzpicture}
		};
		\node (122)  at (5, 1)
		{
			\begin{tikzpicture}[]
				\node (1) at (0,2) [greennode];
                
				\node (2) at (1,3) [rednode];
				\node (3) at (1,1) [cyannode];
                
				\node (4) at (2,3.5) [orangenode];
				\node (5) at (2,2.5) [greennode];
				\node (6) at (2,1.3) [yellownode];
                    \node (7) at (2,0.3) [magentanode];
				
				\draw (1) -- (2) -- (4);
                    \draw (2) -- (5);
                    \draw (1) -- (3) -- (6);
                    \draw (3) -- (7);

				\node (11) at (0,4)[opacity=1] {$L_0$};
				\node (11d) at (0,0.3)[opacity=0] {};
				
				\node (12) at (1,4)[opacity=1] {$L_1$};
                    \node (12b) at (1,2.2)[opacity=1] {};
				\node (12d) at (1,0.3)[opacity=0] {};
				
				\node (13) at (2,4)[opacity=1] {$L_2$};
				\node (13d) at (2,0.3)[opacity=0] {};

				\begin{scope}[on background layer]
					\fill[yellow, opacity=0.3] \convexpath{11, 1, 3, 12d, 11d}{11pt};
					\fill[yellow, opacity=0.3] \convexpath{12,  12b, 6, 13d, 6, 12b}{11pt};
					\fill[yellow, opacity=0.3] \convexpath{13, 5}{11pt};
				\end{scope}
			\end{tikzpicture}
		};

            \draw [arrow]  (111) -- (122)  node [cyanedgenode];
	\end{tikzpicture}
    }
    \caption{\textbf{$LBL$ as heuristic is not consistent:} The graph depicted on the left side has lower bound $LBL$ of $5$ (determined by level $i=2$). After one move of playing color 2 (blue), the lower bound $LBL$ of the resulting game state goes down to $3$ (determined by both, level $i=1$ and $i=2$). The value of a consistent heuristic may only decrease by $1$ after performing one move.}
    \label{fig:lbl_not_consistent}
\end{figure}
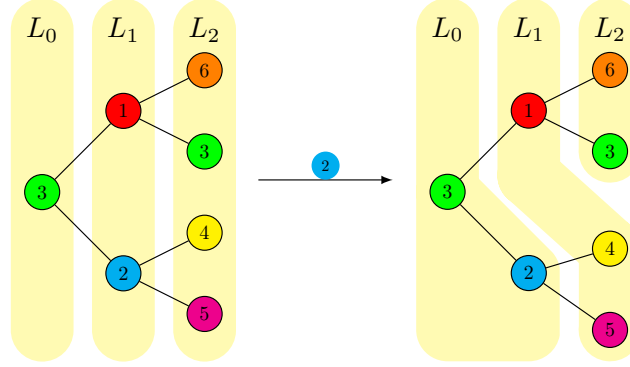

\end{document}